\documentclass{article}

\usepackage{PRIMEarxiv}

\usepackage[utf8]{inputenc} % allow utf-8 input
\usepackage[T1]{fontenc}    % use 8-bit T1 fonts
\usepackage{hyperref}       % hyperlinks
\usepackage{url}            % simple URL typesetting
\usepackage{booktabs}       % professional-quality tables
\usepackage{amsfonts}       % blackboard math symbols
\usepackage{nicefrac}       % compact symbols for 1/2, etc.
\usepackage{microtype}      % microtypography
\usepackage{lipsum}
\usepackage{fancyhdr}       % header
\usepackage{graphicx}       % graphics
\graphicspath{{media/}}     % organize your images and other figures under media/ folder
\usepackage{mathtools}
\usepackage{mlmath}
\usepackage{amsmath}
\usepackage{amssymb}
\usepackage{amsthm}
\usepackage{subcaption}
\usepackage{xcolor}
\usepackage[section]{placeins}
\usepackage{multirow}
\usepackage{enumitem}
\usepackage{zhlipsum}

\newtheorem{thm}{Theorem}[section]

\newtheorem{lem}[thm]{Lemma}

\newtheorem{remark}[thm]{Remark}
\newtheorem{assum}[thm]{Assumption}
\newtheorem{exam}[thm]{Example}

\author{
    Shuyu Yin \footnotemark[2]
    \and
    Fei Wen \footnotemark[2]
    \and
    Peilin Liu \footnotemark[2]
    \and
    Tao Luo \footnotemark[1] \, \footnotemark[3] \, \footnotemark[4]
}

\title{Probing Implicit Bias in Semi-gradient Q-learning: Visualizing the Effective Loss Landscapes via the Fokker--Planck Equation}

\begin{document}
\maketitle

\renewcommand{\thefootnote}{\fnsymbol{footnote}}
\footnotetext[1]{Corresponding Author: luotao41@sjtu.edu.cn} 
\footnotetext[2]{Department of Electronic Engineering, Shanghai Jiao Tong University}
\footnotetext[3]{School of Mathematical Sciences, Institute of Natural Sciences, MOE-LSC, CMA-Shanghai, Shanghai Jiao Tong University} 
\footnotetext[4]{Shanghai Artificial Intelligence Laboratory}

\begin{abstract}
    % semi-gradient Q-learning在许多领域中都有应用，但由于semi-gradient没有对应的explicit loss function，其在参数空间中的dynamics和implicit bias难以研究。本文引入Fokker-Plnack equation，使用采样得到的部分数据，在二维参数空间进行了effective loss landscape的可视化。该可视化展示了loss landscape的global minima可以转化成effective loss landscape的saddle，以及semi-gradient method的implicit bias。最终，我们在高维参数空间和神经网络的设定下展示了effective loss landscape中依然存在由global minima转化过来的saddle point。本文为理解semi-gradient Q-learning在参数空间中的行为提供了新的视角。
    Semi-gradient Q-learning is applied in many fields, but due to the absence of an explicit loss function, studying its dynamics and implicit bias in the parameter space is challenging. This paper introduces the Fokker--Planck equation and employs partial data obtained through sampling to construct and visualize the effective loss landscape within a two-dimensional parameter space. This visualization reveals how the global minima in the loss landscape can transform into saddle points in the effective loss landscape, as well as the implicit bias of the semi-gradient method. Additionally, we demonstrate that saddle points, originating from the global minima in loss landscape, still exist in the effective loss landscape under high-dimensional parameter spaces and neural network settings. This paper develop a novel approach for probing implicit bias in semi-gradient Q-learning.
\end{abstract}

\section{Introduction}
\label{sec::intro}

% Q-learning is a classic reinforcement Learning algorithm, it usually combined with function approximation, i.e., Deep Q-Network (DQN)~\cite{mnih2015human}. This algorithm has applications in a lot of domains, i.e. gaming~\cite{mnih2015human,silver2017mastering,vinyals2019grandmaster}, recommendation system~(\cite{deng2021unified}), combinatorial optimization~(\cite{bello2016neural,khalil2017learning}). The optimization goal of Q-learning is to minimize an empirical Bellman optimal loss \eqref{eq::BellmanOptimalLoss}. Semi gradient \eqref{eq::semiGradient} method is the most common method to minimize the loss. However, semi gradient is not the exact gradient of the loss, it remove the term contains maximum operation. This method can bias against all critical points and diverge~\cite{sutton2018reinforcement, tsitsiklis1996analysis, van2018deep, achiam2019towards}. On the contrary, residual gradient \eqref{eq::trueGradient} method~\cite{baird1995residual}, which is the exact gradient of the loss, is stable. Additionally, when partial data is used to train the model, i.e., mini-batch and replay buffer~\cite{mnih2015human}, these two methods can converge to different policy \cite{saleh2019deterministic, zhang2019deep}, which indicates different implicit bias. Nevertheless, the semi gradient method is much more popular than residual gradient method in practice. Because of the empirical success, we want to study the implicit bias of the semi-gradient method in the parameter space.

Q-learning, a classic Reinforcement Learning (RL) algorithm, is often paired with function approximation, such as the Deep Q-Network (DQN)~\cite{mnih2015human}. This algorithm finds applications in various domains, including gaming~\cite{mnih2015human,silver2017mastering,vinyals2019grandmaster}, recommendation systems~\cite{deng2021unified}, and combinatorial optimization~\cite{bello2016neural,khalil2017learning}. The primary objective of Q-learning is to minimize an empirical Bellman optimal loss. The semi-gradient method is commonly employed to minimize this loss. The semi-gradient approach deviates from the exact gradient by omitting the term involving the maximum operation, converge fast but potentially leading to divergence~\cite{sutton2018reinforcement, tsitsiklis1996analysis, van2018deep, achiam2019towards}. In contrast, the residual gradient method ~\cite{baird1995residual} represents the precise gradient of the loss, it offers stability but converge slowly. Moreover, when training the model with partial data, such as through mini-batch and replay buffer techniques~\cite{mnih2015human}, these two methods may converge to different policies~\cite{saleh2019deterministic, zhang2019deep}. Additionally, the semi-gradient method is more prevalent in practical applications. Motivated by this success, we want to investigate the implicit bias of semi-gradient Q-learning.

% The current research on implicit bias\ref{} contain following directions: the relationship between over-parameterization and generalization\ref{}, properties of the parameters of learned neural networks\ref{}, and different preferences of algorithm for critical points\ref{}. These studies primarily concentrate on supervised learning and rely on gradient flow of a explicit loss function. However, due to the absence of a corresponding explicit loss function for the semi-gradient, employing the analytical methodologies common in supervised learning is unfeasible. Because the semi-gradient does not have correspondent explicit loss function, using the analytical method proposed in supervised learning~\cite{smith2021origin, wu2022alignment, barrett2020implicit, zhang2023structure} is impossible. To deal with this issue, we use Fokker--Planck Equation (FPE)~\cite{wang2008potential, zhou2016construction} to construct the effective loss landscape and then visualize it. In biology and statistical mechanics, FPE is used to model the effective loss landscape for non-conservation force, which means the force is not a negative gradient of any loss function (more information in Appendix \ref{sec::fpeInfo}). The semi-gradient can be regarded as an non-conservation force, so we can use FPE to construct a effective loss landscape for it. In this work, We first construct the effective loss landscape, then demonstrate the implicit bias of semi-gradient method through training dynamics.

The research on implicit bias~\cite{vardi2023implicit} contains following directions: the relationship between over-parameterization and generalization~\cite{belkin2019reconciling}, properties of the parameters of learned neural networks~\cite{ergen2021convex}, and different preferences of algorithm for critical points~\cite{keskar2017large}. These studies primarily concentrate on supervised learning and rely on the gradient flow of a explicit loss function. However, due to the absence of a corresponding explicit loss function for the semi-gradient, employing the analytical methodologies common in supervised learning is unfeasible. To address the challenge, we employ the Fokker--Planck Equation (FPE)~\cite{wang2008potential, zhou2016construction} to establish the effective loss landscape and subsequently visualize it. In fields like biology and statistical mechanics, the FPE is utilized to depict the effective loss landscape in scenarios involving non-conservative forces, where the force does not align with the negative gradient of any specific loss function (further details in Appendix \ref{sec::fpeInfo}). Given that the semi-gradient can be viewed as a non-conservative force, we leverage the FPE to construct an effective loss landscape for it. Our approach involves initially constructing the effective loss landscape and then showcasing the implicit bias of the semi-gradient method through training dynamics.

% A demonstration of the (effective) loss landscape and training dynamics of residual gradient and semi-gradient method is given in Figure \ref{fig::demo_dynamics_semi_true}. Because only partial data is used to construct the loss landscape, there are two solutions in the landscape (orange and blue stars). The training dynamics of residual gradient method in (a) is totally different with the training dynamics of semi-gradient in (b). Besides, the training dynamics 2 for the semi-gradient method diverges and loss increases exponentially as show in (c). Additionally, comparing (a) and (b), the blue star transform from a global minimum to a saddle point. This simple example provide us an insight that the semi-gradient method have different implicit bias against the residual gradient method when trained with partial data.

Figure \ref{fig::demo_dynamics_semi_true} illustrates the (effective) loss landscape and training dynamics associated with both the residual gradient and semi-gradient method. Due to the utilization of only partial data for constructing the loss landscape, two solutions emerge within it, represented by orange and blue stars. Notably, the training dynamics of the residual gradient method in (a) exhibit a stark contrast to those of the semi-gradient method in (b). Furthermore, the training dynamics for the semi-gradient method diverge, showcasing an exponential increase in loss, as depicted in (c). Upon comparing (a) and (b), it becomes apparent that the blue star transitions from a global minimum to a saddle point. This straightforward example offers valuable insights into the distinct implicit biases exhibited by the semi-gradient method compared to the residual gradient method when trained with partial data.

% 加一张蓝色线的loss图，把时间变长，用log scale。
\begin{figure}[htb]
\centering
\begin{subfigure}[t]{.31\textwidth}
  \centering
  % include first image
  \includegraphics[width=1\linewidth]{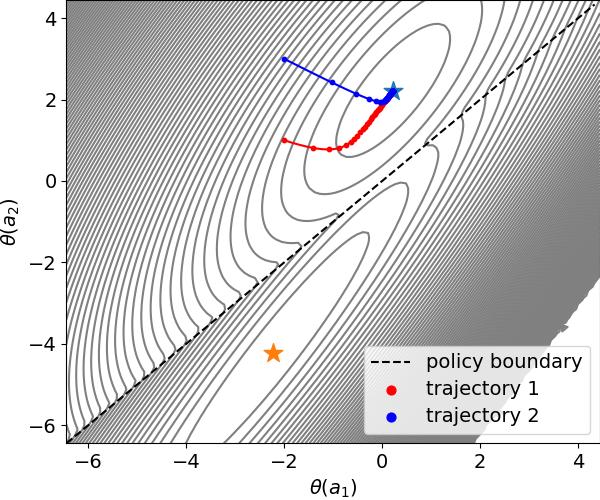}
  \caption{training dynamics of residual gradient method}
  \label{fig:sub-first}
\end{subfigure}
\hspace{10mm}
\begin{subfigure}[t]{.31\textwidth}
  \centering
  % include first image
  \includegraphics[width=1\linewidth]{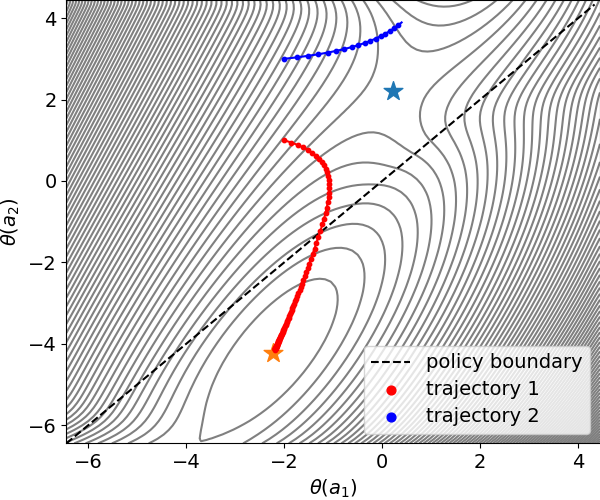}
  \caption{training dynamics of semi-gradient method}
  \label{fig:sub-first}
\end{subfigure}
\caption{
% Training dynamics of the residual gradient and semi-gradient method start from $(-2, 1)$(red) and $(-2, 3)$(blue) with the landscape constructed by data $\{(s_1, a_1, s_2, r)$, $(s_1,a_2,s_3,r)\}$ from Example \ref{exam::simpleMDP}. In (a), both trajectories reach $\theta_{\pi_1}$. In (b), trajectories start from $(-2, 1)$ first approaching $\theta_{\pi_2}$ and then leave $\theta_{\pi_2}$ to $\theta_{\pi_1}$. Besides, the trajectory start from $(-2, 3)$ leaves $\theta_{\pi_2}$ and going to diverge, the loss increases exponentially, as showed in (c).
The training dynamics of two methods initiate from the points $(-2, 1)$ (red) and $(-2, 3)$ (blue) within the landscape constructed using the data $\{(s_1, a_1, s_2, r), (s_1, a_2, s_3, r)\}$ from Example \ref{exam::simpleMDP}. In (a), both trajectories converge to $\theta_{\pi_1}$. Conversely, in (b), all the trajectories bias against $\theta_{\pi_2}$.
}
\label{fig::demo_dynamics_semi_true}
\end{figure}

% The main contribution are three folds: (1) Introduced a new tool to stdy the implicit bias of semi-gradient method, which is Wang's potential landscape theory~\cite{wang2008potential}. (2) Visualized and analyzed the effective loss landscape and the training dynamics of the semi-gradient method in $\mathbb{R}^2$ parameter space. We provide three important insights: first, the semi-gradient method can consider some global minimum in loss landscape as a saddle point when trained with partial data; second, the gradient of state action value function $Q$ can effect the position of the saddle point; third, The training dynamics can diverge around saddle point. 3) Extended the understanding of implicit bias of the semi-gradient method to high dimensional parameter space, i.e., neural networks. % 具体来说，我们提出的new approach包含三个步骤：首先，通过引入可视化工具，基于简单例子观测semi-gradient Q-learning的行为，之后，基于可视化的结果，构建对于该implicit bias的感性理解；最后，根据感性理解在高维情况下进行实验设计和理论解释。

\textbf{Main Contributions:} 1) \textbf{Visualization and Analysis of the Effective Loss Landscape and Implicit Bias in $\mathbb{R}^2$ Parameter Space:} We introduced Wang's potential landscape theory~\cite{wang2008potential} to explore the effective loss landscape of the semi-gradient method. This approach enabled us to visualize the effective loss landscape and help us to analyze the implicit bias. We highlight two key insights: first, the semi-gradient method may transform certain global minima into a saddle point when only partial data is available; second, the gradient of the state-action value function $Q$ can influence the position of the saddle points. 2) \textbf{Extension of Implicit Bias Understanding to Higher Dimensions:} We have shown that the implicit bias observed with the semi-gradient method is also present in high-dimensional parameter spaces and neural networks. This extends our comprehension of implicit bias into more complex and higher-dimensional spaces. 3) \textbf{Development of a Novel Approach for Probing Implicit Bias in Semi-gradient Q-learning:} Our approach comprises three steps: initially, we introduce a visualization tool to detect implicit bias in a simple example; next, we establish an intuitive grasp of the implicit bias; and finally, we design experimental procedures for high-dimensional scenarios to demonstrate the presence of implicit bias. All the code related to this research is available on GitHub at \hyperlink{https://github.com/dayhost/FPE}{https://github.com/dayhost/FPE}.

\section{Related Works}

% Our work is closely related with the comparison study of the residual gradient and semi-gradient method. Schoknecht \textit{et al}~\cite{schoknecht2003td} and Li \textit{et al}~\cite{li2008worst} compares the convergence rate of residual gradient and semi-gradient method in policy evaluation and linear approximation setting. Saleh \textit{et al}~\cite{saleh2019deterministic} compare the policy learn by residual and semi-gradient Q-learning in some some deterministic and stochastic environment. Zhang \textit{et al}~\cite{zhang2019deep} embedded residual gradient method into DDPG~\cite{lillicrap2015continuous} and get a better score in DeepMind Control Suite compare to DDPG with semi-gradient method.
% The last two works imply that residual gradient method can learn a different policy compare to semi-gradient method when trained with partial data.

Our research is closely related with the comparative analysis of the residual and semi-gradient methods. Schoknecht \textit{et al.}~\cite{schoknecht2003td} and Li \textit{et al.}~\cite{li2008worst} have examined the convergence rates of the residual and semi-gradient methods in policy evaluation within a linear approximation framework. Saleh \textit{et al.}~\cite{saleh2019deterministic} have compared the policies learned by residual and semi-gradient Q-learning in both deterministic and stochastic environments. Furthermore, Zhang \textit{et al.}~\cite{zhang2019deep} integrated the residual gradient method into DDPG~\cite{lillicrap2015continuous} and achieved improved performance in the DeepMind Control Suite compared to DDPG utilizing the semi-gradient method. 

% The divergence phenomenon of the semi-gradient method is also related with this work. Sutton \textit{el al}~\cite{sutton2018reinforcement} purposed the classic "Deadly Triad" condition, which is "temporal difference loss", "function approximation" and "off-policy data" can cause semi-gradient method to diverge. Tsitsiklis \textit{et al}~\cite{tsitsiklis1996analysis} provides a concrete example to show that the semi-gradient method with non-linear function approximation can diverge. Achiam \textit{et al}~\cite{achiam2019towards} shows that when trained with whole data if the contraction mapping property of Bellman operator not satisfied, the training dynamic can diverge. 

% The study of the implicit bias is related with this work. A initial investigation for the implicit bias is provided by \cite{}, it demostrate the capcity control property of neural network in supervised learning, which can lead to better generalization. 类似的，\cite{}发现了double descent的现象，说明神经网络在over-parameterized的情况下随着参数量的增加generalization会更好。\citr{} studied the property of the two-layer neural networks trained with reguliazed loss, which is the implicit bias of the regulization term. \cite{}发现了SGD方法往往会找到更加平坦的局部最小点，这增加了神经网络的泛化能力。然而，当前对于implicit bias的研究主要集中于supervised learning，本文希望在强化学习领域中讨论implicit bias。
The investigation of implicit bias is relevant to this study. An initial exploration of implicit bias was conducted by Neyshabur \textit{el al} \cite{neyshabur2014search}, it demonstrates that the capacity-controlling property of neural networks can lead to improved generalization. Similarly, Belkin \textit{el al} \cite{belkin2019reconciling} discovered the phenomenon of double descent, indicating that neural networks in over-parameterized regime exhibit better generalization as the number of parameters increases. Ergen \textit{el al} \cite{ergen2021convex} analyzes the properties of critical points of two-layer neural networks with regularized loss. Keskar \textit{et al} \cite{keskar2017large} found that the stochastic gradient descent (SGD) method often converges to flatter local minima, enhancing the generalization ability of neural networks. Current research on implicit bias predominantly focuses on supervised learning, while this work aims to discuss implicit bias in the realm of reinforcement learning.

% The method for Non-equilibrium loss landscape modeling is related with this work. Wang \textit{et al}~\cite{wang2008potential} provides  Wang's potential landscape theory for effective loss landscape construction, we use this method in this work. Zhou \textit{et al}~\cite{zhou2016construction} summaries three methods to construct effective loss landscape, which are Wang's Potential landscape theory, Freidlin-Wentzell quasi-potential method and A-type integral method.

The methodology for modeling non-equilibrium loss landscapes is also relevant to our research. Wang \textit{et al.}~\cite{wang2008potential} introduced Wang's potential landscape theory for constructing effective loss landscapes, a method we employ in our work. Additionally, Zhou \textit{et al.}~\cite{zhou2016construction} outlined three approaches for constructing effective loss landscapes, including Wang's potential landscape theory, the Freidlin-Wentzell quasi-potential method, and the A-type integral method.

\section{Effective Loss Landscape Visualization and Implicit Bias Demonstration}
\label{sec::two_dimension}

% In order to discuss the loss landscape, We first give the definition of empirical Bellman optimal loss, residual gradient and semi-gradient. The empirical Bellman optimal loss is defined as
% \begin{equation}
%     \label{eq::BellmanOptimalLoss}
%     \mathcal{L} = \frac{1}{|\fD|}\sum_{(s,a,s',r) \in \fD} \left( Q(s,a) - r(s,a) - \gamma \max_{a' \in A} Q(s',a') \right)^2,
% \end{equation}
% where $s \in S$ is a state, $S$ is the state space, $a \in A$ is an action, $A$ is the action space, $r: S \times A \to \mathbb{R}$ is the reward function, $\fD$ is the sample data set, $Q: S \times A \to \mathbb{R}$ is the state value function, $\gamma \in (0,1)$ is the discount factor. The semi-gradient is defined as
% \begin{equation}
%     \label{eq::semiGradient}
%     \nabla \mathcal{L}_{\rm{semi}} = \frac{2}{|\fD|}\sum_{(s,a,s',r) \in \fD} \nabla Q(s,a) \left( Q(s,a) - r(s,a) - \gamma \max_{a' \in A} Q(s',a') \right).
% \end{equation}
% The residual gradient is defined as follow
% \begin{equation}
%     \label{eq::trueGradient}
%     \nabla \mathcal{L}_{\rm{res}} = \frac{2}{|\fD|}\sum_{(s,a,s',r) \in \fD} \left( \nabla Q(s,a) - \nabla \max_{a' \in A} Q(s',a') \right) \left( Q(s,a) - r(s,a) - \gamma \max_{a' \in A} Q(s',a') \right).
% \end{equation}

Before dive into the loss landscape, we initially define the empirical Bellman optimal loss, residual gradient, and semi-gradient. The empirical Bellman optimal loss is defined as $\mathcal{L} = \frac{1}{|\mathcal{D}|}\sum_{(s,a,s',r) \in \mathcal{D}} \left( Q(s,a) - r(s,a) - \gamma \max_{a' \in A} Q(s',a') \right)^2$, where $s \in S$ represents a state, $S$ denotes the state space, $a \in A$ signifies an action, $A$ denotes the action space, $r: S \times A \to \mathbb{R}$ symbolizes the reward function, $\mathcal{D}$ denotes the sample dataset, $Q: S \times A \to \mathbb{R}$ stands for the state value function, and $\gamma \in (0,1)$ represents the discount factor. The semi-gradient is defined as $\nabla \mathcal{L}_{\rm{semi}} = \frac{2}{|\mathcal{D}|}\sum_{(s,a,s',r) \in \mathcal{D}} \nabla Q(s,a) \left( Q(s,a) - r(s,a) - \gamma \max_{a' \in A} Q(s',a') \right).$
The residual gradient is defined as $\nabla \mathcal{L}_{\rm{res}} = \frac{2}{|\mathcal{D}|}\sum_{(s,a,s',r) \in \mathcal{D}} \left( \nabla Q(s,a) - \nabla \max_{a' \in A} Q(s',a') \right) $ $\left( Q(s,a) - r(s,a) - \gamma \max_{a' \in A} Q(s',a') \right).$

\subsection{Setting and one solution scenario}
\label{sec::two_dimension_setting}
In this subsection, we present an example to illustrate the visualization of loss landscapes. All figures depicting loss landscapes in this section are based on the settings outlined in Example \ref{exam::simpleMDP}.

\begin{exam}[example for visualization]
    \label{exam::simpleMDP}
    % Consider a deterministic Markov Decision Process (MDP) $\fM(S,A,f,r,\gamma)$ with three states $S = \{s_1,s_2,s_3,s_4\}$ and two actions $A = \{a_1,a_2\}$. State $s_4$ is a terminal state and the transition function $f: \{s_1,s_2,s_3\} \times \{a_1,a_2\} \to \{s_1,s_2,s_3,s_4\}$ is demonstrate in Figure \ref{fig:concrete_mdp_example}. The reward function $r: \{s_1,s_2,s_3\} \times \{a_1,a_2\} \to \mathbb{R}$ is defined on each transition path. The discount factor is $\gamma \in (0, 1)$. Specifically, we set $\gamma=0.9$ and $r(s, a)=-0.1$ for all $(s, a) \in \{s_1, s_2, s_3\} \times \{a_1, a_2\}$. $r(s,a)$ is simplified as $r$ in the following discussion.
    Consider a deterministic Markov Decision Process (MDP) $\mathcal{M}(S,A,f,r,\gamma)$ with three states $S = \{s_1, s_2, s_3, s_4\}$ and two actions $A = \{a_1, a_2\}$. State $s_4$ serves as a terminal state, and the transition function $f: \{s_1, s_2, s_3\} \times \{a_1, a_2\} \to \{s_1, s_2, s_3, s_4\}$ is illustrated in Figure \ref{fig:concrete_mdp_example}. The reward function $r: \{s_1, s_2, s_3\} \times \{a_1, a_2\} \to \mathbb{R}$ is defined along each transition path. The discount factor is denoted as $\gamma \in (0, 1)$. Specifically, we fix $\gamma=0.9$ and set $r(s, a)=-0.1$ for all $(s, a) \in \{s_1, s_2, s_3\} \times \{a_1, a_2\}$. For simplicity, we define $r:=r(s,a)$.
\end{exam}

\begin{figure}[htb] 
\centering
\includegraphics[width=0.3\linewidth]{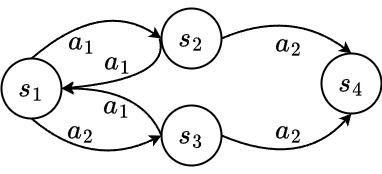}
\caption{The geometric structure of the given MDP environment. }
\label{fig:concrete_mdp_example}
\end{figure}

\textbf{Parameterization of $Q$.} 
% We mainly use a linear model with two parameters to approximate the $Q$ function, which is $\theta = [\theta(a_1), $ $\theta(a_2) ]^{\text{T}}$ and $\theta(a_1), \theta(a_2) \in \mathbb{R}$, to parameterize the state action value function $Q: S \times A \to \mathbb{R}$. The function can be represented by the multiplication of state embedding and parameter as $Q(s,a) = \phi(s) \theta(a)$, where $\phi(s) \in \mathbb{R}$ is the state embedding of state $s$. Specifically, we assume $\phi(s_1)=0.1$, $\phi(s_2)=11/180$, $\phi(s_3)=29/180$ and reward function is constant $-0.1$. Because all the state embedding are positive, there are only two policies which can be defined as $\pi_1(s_1) = a_1, \pi_1(s_2) = a_1$ and $\pi_2(s_1)=a_2, \pi_2(s_2)=a_2$. $\theta(a_1)=\theta(a_2)$ is the policy boundary.
Our primary approach involves employing a linear model with two parameters to approximate the $Q$ function. These parameters are denoted as $\theta = [\theta(a_1), \theta(a_2)]^{\text{T}}$, where $\theta(a_1), \theta(a_2) \in \mathbb{R}$, to parameterize $Q$. This function can be expressed as the product of the state embedding and parameters: $Q(s,a) = \phi(s) \theta(a)$, where $\phi(s) \in \mathbb{R}$ represents the state embedding of state $s$. Specifically, we assume $\phi(s_1)=0.1$, $\phi(s_2)=11/180$ and $\phi(s_3)=29/180$. Given that all state embeddings are positive, only two policies can be defined: $\pi_1(s_1) = a_1, \pi_1(s_2) = a_1$ and $\pi_2(s_1)=a_2, \pi_2(s_2)=a_2$. The equality $\theta(a_1)=\theta(a_2)$ delineates the policy boundary.

\begin{remark}
    % In most theoretical analysis of Q-learning with linear approximation, the state value function is represented as $Q(s,a) =\phi(s,a) \theta$, where $\phi(s,a) \in \sR^d, \theta \in \sR^d$. However, in the common practice of deep Q-learning, the input data are state embedding, which only depend in state, the output is a vector represents each action value for the given state. In order to make the model close to the common practice, we use the model above. Besides, we assume $\phi(s_1), \phi(s_2), \phi(s_3) > 0$. This is because we can consider the output of second last layer in a neural network using ReLU as activation as the state embedding, the element are non-negative.
    In many theoretical analyses of Q-learning with linear approximation, the state-value function is typically denoted as $Q(s,a) = \phi(s,a) \theta$, where $\phi(s,a) \in \mathbb{R}^d$ and $\theta \in \mathbb{R}^d$. However, in the practical implementation of DQN, the input data consists of state embeddings that depend solely on the state, while the output is a vector representing the action values for the given state. We adopt the setting in DQN. Additionally, we assume that $\phi(s_1), \phi(s_2), \phi(s_3) > 0$. This assumption stems from the interpretation that the output of the second-to-last layer in a neural network utilizing ReLU activation can be considered as the state embedding, where the elements are non-negative.
\end{remark}

% \begin{remark}
%     From the definition of the force of True Gradient, if $\phi(s_1) \neq \gamma \phi(s_2)$ there is no local minimum in its loss landscape. However, if $\phi(s_1) = \gamma \phi(s_2)$, there will be infinitely many local minimums with $\theta(a_1) < \theta(a_2)$ and $\theta(a_1) = r(s_1,a_1)/\phi(s_1) + \gamma \phi(s_2)$. In addition, from the definition of semi-gradient, there is no local minimum.
% \end{remark}

\textbf{Settings for numerical calculation.} 
% The "force" for the Fokker--Planck equation is the negative residual or semi-gradient with different policy, i.e. \eqref{eq::trueGradientPi1},\eqref{eq::trueGradientPi2},\eqref{eq::semiGradientPi1},\eqref{eq::semiGradientPi2}. In order to use the numerical method\footnote{Github: https://github.com/johnaparker/fplanck} to solve the equation, the "force" is discretized into a force matrix, the size of this matrix is $100 \times 100$. Besides, the distribution is also discretized into a $100 \times 100$ matrix. The resolution of this discretization is $0.11$ and the propagation time to find the stationary distribution is $100000$. The loss landscape calculated by Fokker--Planck equation also depends on a diffusion constant $\sigma$, when $\sigma \to 0$ the effective loss landscape is approaching the real static effective loss landscape. However, smaller diffusion constant requires much larger computational resources. Therefore, we choose a small enough confusion constant $\sigma=2^{-8}$. The discretization and non-zero diffusion constant introduce a tolerable numerical error to the visualization. 
The "force" in the Fokker--Planck equation is the negative residual or semi-gradient associated with different policies, i.e., \eqref{eq::trueGradientPi1}, \eqref{eq::trueGradientPi2}, \eqref{eq::semiGradientPi1}, and \eqref{eq::semiGradientPi2}. To facilitate the solution of the equation using the numerical method\footnote{Github:\hyperlink{ https://github.com/johnaparker/fplanck}{ https://github.com/johnaparker/fplanck}}~\cite{holubec2019physically}, this "force" is discretized into a force matrix with dimensions of $100 \times 100$. Additionally, the probability distribution is discretized into a matrix with same size. The resolution of this discretization is $0.095$, and the propagation time required to determine the stationary distribution is $100,000$. The loss landscape, computed via the Fokker--Planck equation, is also influenced by a diffusion constant $\sigma$. As $\sigma \to 0$, the effective loss landscape converges towards the true static effective loss landscape. However, a smaller diffusion constant necessitates significantly greater computational resources. Hence, we opt for a sufficiently small diffusion constant, specifically $\sigma=2^{-8}$. The discretization process and the presence of a non-zero diffusion constant introduce a manageable level of numerical error to the visualization. Besides, we use an NVIDIA GeForce RTX 3080 GPU to do numerical calculation.

\textbf{Data sampling strategy.}
% For each figure, we select a mini-batch consisting of two samples with $a_1$ and $a_2$ from the MDP, denoted as $\{(s_{\alpha},a_1,s'_{\alpha},r), (s_{\beta},a_2,s'_{\beta},r)\}$. Then we can form up to nine distinct combinations (mini-batches). Five mini-batches have a single solution, and four mini-batches have two solutions. We define the solution in area $\{\theta \in \mathbb{R}^2|\theta(a_1) \geq \theta(a_2)\}$ as $\theta_{\pi_1}$ and the solution in area $\{\theta \in \mathbb{R}^2|\theta(a_1) \leq \theta(a_2)\}$ as $\theta_{\pi_2}$. The existence condition of these two solutions is given in Lemma \ref{lem::existenceOfGloablMin}.
% 在实际应用中，训练的数据需要从环境中采样得到，因此往往只包含部分环境数据。同时，不同的采样数据会影响到loss landscape和training dynamics。为了展示实际情况中不同数据的(effective) loss landscape，我们设计如下的采样策略：从所有包含$a_1$和$a_2$的数据中分别采样一个数据组成一个mini-batch，denoted as $\{(s_{\alpha},a_1,s'_{\alpha},r), (s_{\beta},a_2,s'_{\beta},r)\}$，之后通过该mini-batch中的数据进行可视化。当前的采样策略总共有9中可能的mini-batch，其中5个mini-batches有一个solution，4个mini-batch有两个solution。可以看到当前采样策略下，比较容易采样到存在多个solution的mini-batch。
In practical applications, training data is sampled from the environment, often containing only a subset of environmental's data. Moreover, different sample data can influence the loss landscape and training dynamics. To demonstrate the (effective) loss landscape of different data, we designed the following sampling strategy: we sample one data point each containing $a_1$ and $a_2$, forming a mini-batch denoted as $\{(s_{\alpha},a_1,s'_{\alpha},r), (s_{\beta},a_2,s'_{\beta},r)\}$, and visualize it. The current sampling strategy yields a total of nine possible mini-batches, with five mini-batches having one solution and four mini-batches having two solutions. The figures of the loss landscape that were not presented in the main content are displayed in Appendix \ref{sec::addition_figs}. We define the solution within the region $\{\theta \in \mathbb{R}^2|\theta(a_1) \geq \theta(a_2)\}$ as $\theta_{\pi_1}$ and the solution within the region $\{\theta \in \mathbb{R}^2|\theta(a_1) \leq \theta(a_2)\}$ as $\theta_{\pi_2}$. The conditions for the existence of these two solutions are outlined in Lemma \ref{lem::existenceOfGloablMin}. 

\textbf{One solution scenario and smoothness of effective loss landscape.} 
% Figure \ref{fig::loss_landscape_of_one_solution} is representative for loss landscape with one solution. This figure is constructed with data $\{(s_1, a_1, s_2, r)$, $(s_2,a_2,s_4,r)\}$. x-axis is $\theta(a_1)$ and y-axis is $\theta(a_2)$. The definition of \textit{force}, \textit{flux} and \textit{gradient} can be found in Appendix \ref{sec::fpeInfo}. Because $ \frac{\phi(s_{1})-\gamma \phi(s_{2})}{ \phi(s_{2})-\gamma \phi(s_{4})} \approx 0.74$, only the $\theta_{\pi_2}$ exists. Comparing (a) and (b), the contour has a "heart shape" in (a), which indicates the loss landscape is effected by the policy boundary. On the other hand, the contour in (b) is not effected by the policy boundary. From the shape of contour, we can tell that the loss landscape is non-smooth and effective loss landscape is smooth, which is also supported by Lemma \ref{lem::continuitySemi}.
% 当仅有一个critical point的时候，在loss landscape中该critical point为global minima，而在effective loss landscape该critical point可能为一个global minima (如图\ref{})，也可能是一个saddle point(如图\ref{})。 但由于仅仅存在一个critical point的情况在高维参数空间中几乎不存在，因此这里仅仅给出一个示例来展示effective loss landscape的smoothness。
When there is only one critical point, it will be a global minimum in loss landscape, while in the effective loss landscape, this critical point could potentially be a global minimum (as shown in Figure \ref{fig::loss_landscape_of_one_solution}), or a saddle point (as shown in Figure \ref{fig::loss_landscape_s2_a1_s1_s1_a2_s3}). However, since the scenario of having only one critical point is rare in high-dimensional cases, we only provide one example to illustrate the smoothness of the effective loss landscape. Figure \ref{fig::loss_landscape_of_one_solution} is generated using the mini-batch $\{(s_1, a_1, s_2, r), (s_2, a_2, s_4, r)\}$. $\theta(a_1)$ is the abscissa, and $\theta(a_2)$ is the ordinate. In the Figure \ref{fig::loss_landscape_of_one_solution}, the term "force" denotes the negative residual or the semi-gradient. The "gradient" refers to the negative gradient of both the loss and the effective loss landscapes. Within the context of the loss landscape, "force" is equivalent to "gradient". The "flux" is defined as the discrepancy between "force" and "gradient". Nonetheless, the "flux" within the loss landscape is non-zero, attributable to numerical error. Detailed definitions are given in Appendix \ref{sec::fpeInfo}. A comparison between (a) and (b) reveals that the contour in (a) exhibits a "heart shape," indicating the influence of the policy boundary on the loss landscape. In contrast, the contour in (b) is unaffected by the policy boundary. The shape of the contour suggests a non-smooth loss landscape, while the effective loss landscape is smooth, a notion supported by Lemma \ref{lem::continuitySemi}.

\begin{figure}[htb]
\centering
\begin{subfigure}[t]{.31\textwidth}
  \centering
  % include first image
  \includegraphics[width=1\linewidth]{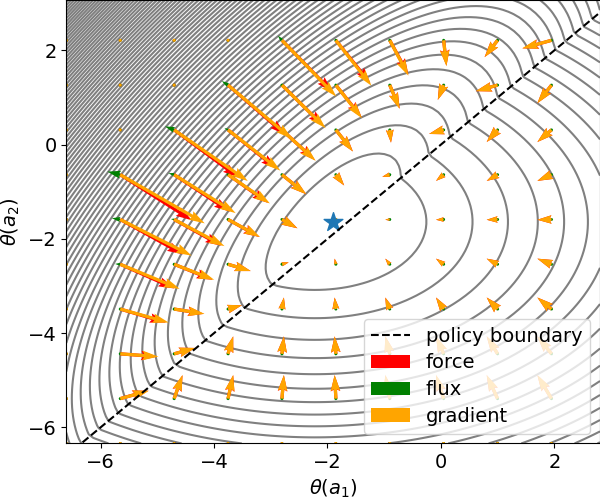}
  \caption{loss landscape}
  \label{fig:sub-first}
\end{subfigure}
\hspace{10mm}
\begin{subfigure}[t]{.31\textwidth}
  \centering
  % include first image
  \includegraphics[width=1\linewidth]{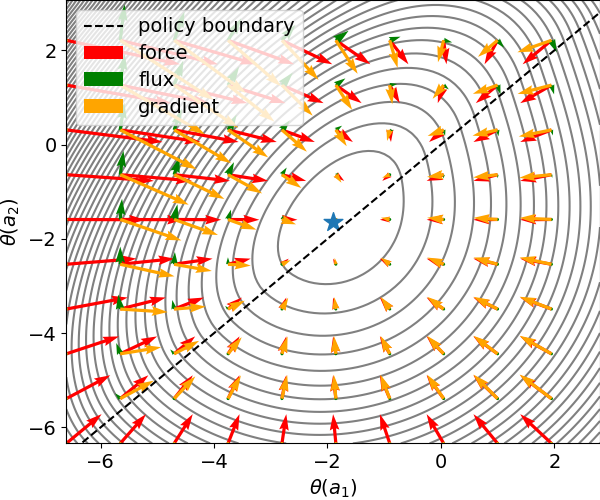}
  \caption{effective loss landscape}
  \label{fig:sub-first}
\end{subfigure}
\centering
\caption{
% Loss landscape with data $\{(s_1, a_1, s_2, r)$, $(s_2,a_2,s_4,r)\}$, only $\theta_{\pi_2}$ exist. In (a), the area around the global minimum in loss landscape is a "heart" shape and non-smooth, because of the policy boundary. However, the effective loss landscape in (b) is smooth. Even though (a) and (b) have different smoothness, the landscape shape is similar.
Loss landscape with the mini-batch $\{(s_1, a_1, s_2, r)$, $(s_2,a_2,s_4,r)\}$ and only $\theta_{\pi_2}$ exists. In (a), the region surrounding the global minimum exhibits a "heart" shape and lacks smoothness. Conversely, the effective loss landscape in (b) is smooth. Despite the difference of smoothness between (a) and (b), their landscape shapes are similar.
}
\label{fig::loss_landscape_of_one_solution}
\end{figure}
% 在第一个例子说一下smoothness，给出定理

% For mini-batches with one solution, the effective loss landscape resembles the one depicted in Figure \ref{fig::loss_landscape_of_one_solution}. However, the effective loss landscape differs significantly when there are two solutions present. In the subsequent subsection, we will concentrate on the loss landscape with two solutions and elucidate the differences.

\subsection{Effective loss landscapes with two solutions}
% 上来先说一下最基本的设定/背景，以及和one solution的区别，也就是发生了什么
% 然后说一下两维上的理解，可以利用定理
% 最后说一下对于高维的猜测。
In this section, we further the discussion utilizing the parameters defined in Example \ref{assum:mdpAssum} and illustrate the notable distinction between the loss landscapes associated with two solutions.

% 本节依然使用example \ref{}中的设定进行讨论。我们考虑两个mini-batch进行可视化，分别为$\{(s_1,a_1,s_2,r)$,$(s_1,a_2,s_3,r)\}$和$\{(s_2,a_1,s_1,r)$,$(s_2,a_2,s_4,r)\}$。对于第一个mini-batch，由于$\phi(s_1) - \gamma \phi(s_2) >0$以及$\phi(s_1) - \gamma \phi(s_3) < 0$，从Lemma \ref{}可知，给定该数据时Bellman Optimal Loss具有两个solution，对应的可视化图为\ref{}。在(a)中，从contour可以看出，$\theta_{\pi_1}$和$\theta_{\pi_2}$(orange and blue star)为两个global minimum，并且policy boundary将两个global minimum切分开了，两个solution之间没有连通性。然而，对于effective loss landscape (b), 从contour可以看出$\theta_{\pi_1}$依然是一个global minimum，但是$\theta_{\pi_2}$转变成了一个saddle，并且两个solution之间具有连通性，这个连通性没有被policy boundary破坏。
\textbf{Transition of global minima to saddle point.} For the mini-batch $\{(s_1, a_1, s_2, r)$, $(s_1,a_2,s_3,r)\}$, since $\phi(s_1) - \gamma \phi(s_2) > 0$ and $\phi(s_1) - \gamma \phi(s_3) < 0$, according to Lemma \ref{lem::existenceOfGloablMin}, the Bellman optimal loss has two solutions. Figure \ref{fig::loss_landscape_of_two_solution1} is constructed with the mini-batch. In Figure \ref{fig::loss_landscape_of_two_solution1} (a), the two solutions $\theta_{\pi_1}$ and $\theta_{\pi_2}$ (orange and blue stars) are two global minima, and the policy boundary separates these two global minima. However, in Figure \ref{fig::loss_landscape_of_two_solution1} (b), the contours reveal that while $\theta_{\pi_1}$ remains a global minimum, $\theta_{\pi_2}$ transitions into a saddle point. The position of the saddle point is consistent with statement $(1)$ of Theorem \ref{thm::SaddPointRange}.

\begin{figure}[htb]
\centering
\begin{subfigure}[t]{.31\textwidth}
  \centering
  % include first image
  \includegraphics[width=1\linewidth]{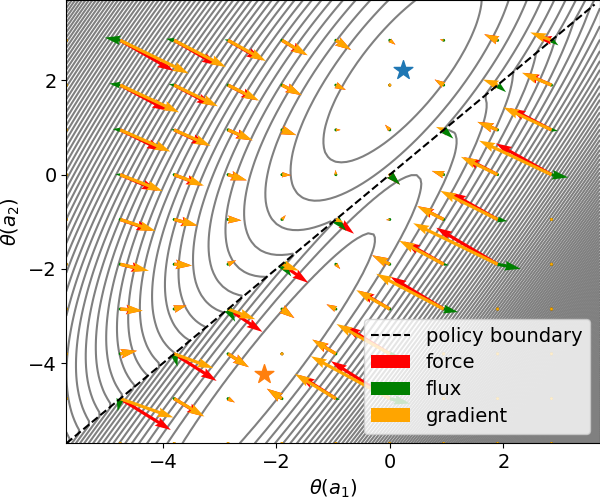}
  \caption{loss landscape}
  \label{fig:sub-first}
\end{subfigure}
\hspace{10mm}
\begin{subfigure}[t]{.31\textwidth}
  \centering
  % include first image
  \includegraphics[width=1\linewidth]{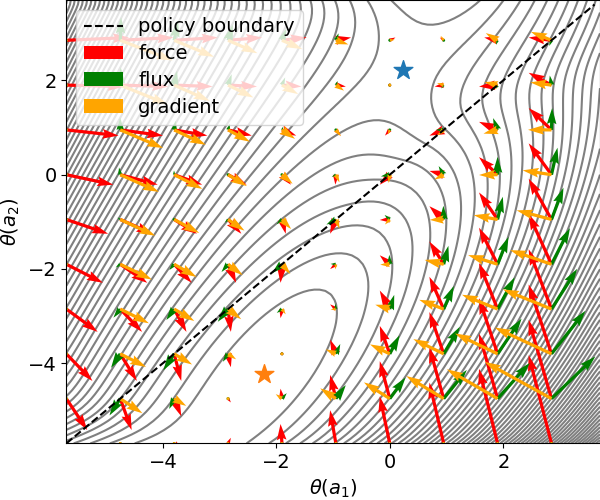}
  \caption{effective loss landscape}
  \label{fig:sub-first}
\end{subfigure}
\caption{
Loss landscapes with the mini-batch $\{(s_1, a_1, s_2, r)$, $(s_1,a_2,s_3,r)\}$, both $\theta_{\pi_1}$ (orange star) and $\theta_{\pi_2}$ (blue star) exist. In (a), two exact solutions are considered as two global minima. However in (b), $\theta_{\pi_1}$ is considered as a global minima but $\theta_{\pi_2}$ is considered as a saddle point.
}
\label{fig::loss_landscape_of_two_solution1}
\end{figure}

% 对于第二个mini-batch，由于$\phi(s_2) - \gamma \phi(s_1) <0$以及$\phi(s_2) - \gamma \phi(s_4) > 0$，根据Lemma \ref{}, 对应的loss landscape上有两个solution，对应图片\ref{}。在(a)中，依然存在两个global minimum，但是在(b)中根据contour可以看出$\theta_{\pi_1}$(orange star)从global minimum变成了saddle。对比图\ref{} (b)和图\ref{} (b)可以看出，saddle point从$\theta_{\pi_2}$变成了$\theta_{\pi_1}$。
\textbf{Displacement of the saddle point.} For the mini-batch $\{(s_2, a_1, s_1, r)$, $(s_2,a_2,s_4,r)\}$, given that $\phi(s_2) - \gamma \phi(s_1) < 0$ and $\phi(s_2) - \gamma \phi(s_4) > 0$, by Lemma \ref{lem::existenceOfGloablMin} there are two solutions. Figure \ref{fig::loss_landscape_of_two_solution2} is constructed with the mini-batch. In Figure \ref{fig::loss_landscape_of_two_solution2} (a), two global minima still exist, but in Figure \ref{fig::loss_landscape_of_two_solution2} (b), the contours reveal that $\theta_{\pi_1}$ (represented by the orange star) became a saddle point. A comparison between Figure \ref{fig::loss_landscape_of_two_solution1} (b) and Figure \ref{fig::loss_landscape_of_two_solution2} (b) shows that the saddle point has displaced from $\theta_{\pi_2}$ to $\theta_{\pi_1}$.  The position of the saddle point is consistent with statement $(2)$ of Theorem \ref{thm::SaddPointRange}.

\begin{figure}[htb]
\centering
\begin{subfigure}[t]{.31\textwidth}
  \centering
  % include first image
  \includegraphics[width=1\linewidth]{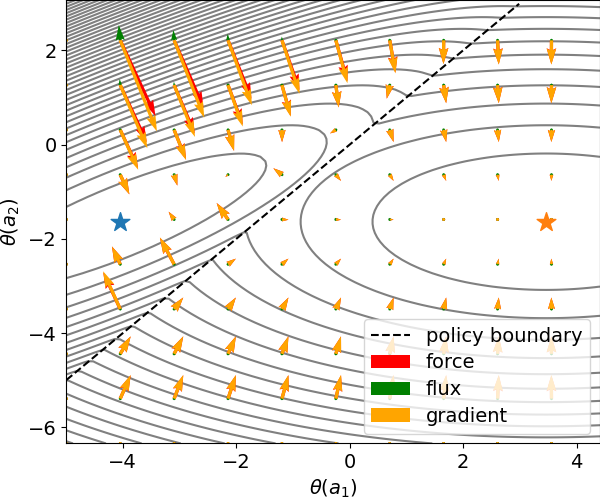}
  \caption{loss landscape}
  \label{fig:sub-first}
\end{subfigure}
\hspace{10mm}
\begin{subfigure}[t]{.31\textwidth}
  \centering
  % include first image
  \includegraphics[width=1\linewidth]{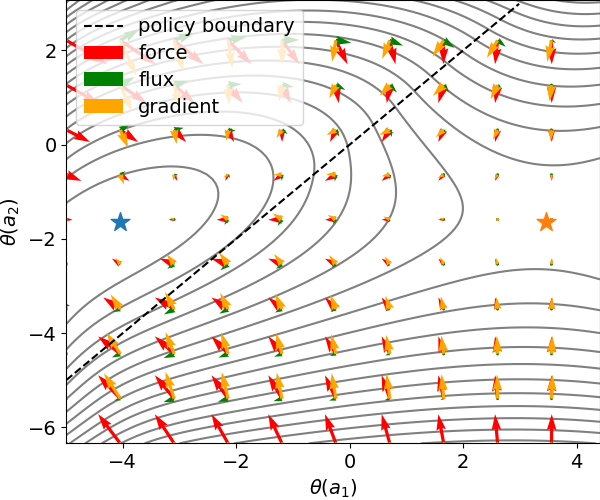}
  \caption{effective loss landscape}
  \label{fig:sub-first}
\end{subfigure}
\caption{Loss landscapes with the mini-batch $\{(s_2, a_1, s_1, r)$, $(s_2,a_2,s_4,r)\}$, both $\theta_{\pi_1}$ (orange star) and $\theta_{\pi_2}$ (blue star) exist. Compare with Figure \ref{fig::loss_landscape_of_two_solution1}, the saddle point shifted from $\theta_{\pi_2}$ to $\theta_{\pi_1}$.}
\label{fig::loss_landscape_of_two_solution2}
\end{figure}

\textbf{Intuitive understanding of the existence of saddle point.} 
% 我们可以先考虑两维的情况。直观来说，由于当前policy boundary两侧都存在solution，并且effective loss landscape上不存在其他的critical point，effective loss landscape的smoothness会导致两个solution之间具有“连通性”，即可以从一个solution达到另外一个solution，这种连通性导致了saddle point的出现。 我们可以进一步推测高维情况下saddle point的存在性。当某一policy boundary的两侧都有solution，并且附近没有其他的critical point的时候，saddle point就会出现，这一理解也在高维的实验(图\ref{}和图\ref{})上得到了验证。
% 直观来说，当在参数空间中的一片区域内semi-gradient和residual gradient的夹角都大于90度，或者说cosine similarity小于0，并且当learning rate足够小的时候semi-gradient能够一直维持在这个区域内时，
Here, our focus lies on the presence of saddle points when multiple solutions exist in the effective loss landscape, a common occurrence in high-dimensional parameter spaces. Let's first consider the two-dimensional scenario. Intuitively speaking, due to the presence of solutions on both sides of the policy boundary and the absence of other critical points on the effective loss landscape, the smoothness of the effective loss landscape results in a "connectivity" between the two solutions, allowing for a path from one solution to the other.  This connectivity gives rise to the emergence of saddle points.. We can further speculate on the existence of saddle points in higher dimensions based on the above understanding. When solutions exist on both sides of a certain policy boundary and there are no other critical points nearby, saddle points will emerge. As shown in Figure \ref{fig::implicit_bias_true_2_semi_dqn} (a) and Figure \ref{fig::dynamics_grid_world} (a), the trajectory need to go cross the policy boundary to converge to another critical point.

\subsection{Divergence and implicit bias of the semi-gradient method}
\label{sec::two_dimension_dynamics}

After demonstrating the effective loss landscape and comparing its differences with the loss landscape, we are prepared to unveil the implicit bias of the semi-gradient method in both two-dimensional and high-dimensional scenario.

% \textbf{Divergence of the semi gradient in two dimension.} 
% % 图\ref{}(b)中的dynamics 2会发散，其对应的loss dynamics如图\ref{}(c)所示。此training dynamics发散是由于$\theta_{\pi_2]$变成了一个saddle point，并且effective loss landscape的右上角不存在其他的critical point使semi gradient收敛。因此，semi gradient会向右上角不断地前进，从(c)中可以看出损失函数会以指数形式增加，导致最终的发散。导致semi gradient出现当前类型的发散不仅仅需要出现saddle point，还需要保证远离saddle point的各个negative semi gradient的方向上没有其他的critical point。然而，第二个条件在高维度，尤其是神经网络的情况下不易发生，在实验中我们也并未观察到此种类型的发散现象。
% In Figure \ref{fig::demo_dynamics_semi_true}(b), dynamics 2 diverges. This phenomena happens because $\theta_{\pi_2}$ transitions into a saddle point, and there are no other critical points in the upper right direction for the effective loss landscape. The occurrence of such a divergence not only necessitates the presence of a saddle point but also requires ensuring that there are no other critical points along the direction. However, the latter condition is challenging to achieve in high dimension. We have not observed this type of divergence phenomenon in high dimension.

\textbf{Semi-gradient method bias against a saddle point.} 
% 图\ref{}由数据$\{(s_1, a_1, s_2, r)$, $(s_1,a_2,s_3,r)\}$产生，并使用section \ref{}给出的线性模型拟合$Q$。给定恒定的学习率$0.1$，我们首先使用residual gradient gradient($t_1$时刻)，使其收敛到$\theta_{\pi_1}$，之后更换为semi gradient gradient训练25000步，使其收敛。(a)展现了loss landscape以及两种不同method下的training dynamics，可以看到当更换为semi gradient method之后，training dynamics跳出了$\theta_{\pi_2}$，在$t_2$时刻穿过了policy boundary并且收敛到了$\theta_{\pi_1}$。这同样表明了$\theta_{\pi_2}$为一个saddle point。在(b)中，当更换为semi gradient method之后，state value发生了明显的变化。结合(a)和(c)可以看出当替换为semi gradient method之后，在training dynamics跨越policy boundary时loss取得最大值。
Here, we start from the two-dimensional scenario. Figure \ref{fig::implicit_bias_true_2_semi} is generated with mini-batch $\{(s_1, a_1, s_2, r), (s_1, a_2, s_3, r)\}$. With a fixed learning rate of $0.1$ and a start point $(-2,1)$, we initially train the model using the residual gradient descent(red) for 25,000 steps ($t_1$). Afterwards, we switch to the semi-gradient descent(blue) and continue training for another 25,000 steps. Figure \ref{fig::implicit_bias_true_2_semi} (a) displays the loss landscape and the training dynamics. It is noticeable that after transit to the semi-gradient method, the training dynamics escape from $\theta_{\pi_2}$, cross the policy boundary at time $t_2$, and converge to $\theta_{\pi_1}$. This behavior also implies that $\theta_{\pi_2}$ is a saddle point for effective landscape. In Figure \ref{fig::implicit_bias_true_2_semi} (b), a significant alteration in the state value is observed after the switch to the semi-gradient method. Combining Figure \ref{fig::implicit_bias_true_2_semi} (a) and (c), it is evident that the loss reaches its apex when the training dynamics cross the policy boundary. The training trajectory of the semi-gradient descent in Figure \ref{fig::implicit_bias_true_2_semi} (a) consist with the statement given in Theorem \ref{thm::SaddPointRange}.

\begin{figure}[htb]
\centering
\begin{subfigure}[t]{.31\textwidth}
  \centering
  % include first image
  \includegraphics[width=1\linewidth]{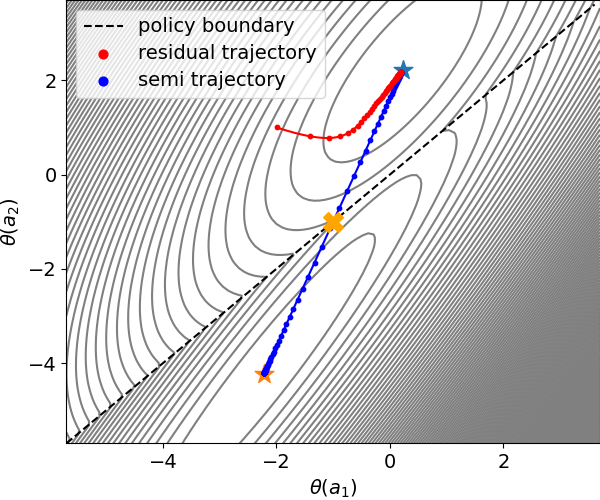}
  \caption{trajectories}
  \label{fig:sub-first}
\end{subfigure}
\hspace{2mm}
\begin{subfigure}[t]{.31\textwidth}
  \centering
  % include first image
  \includegraphics[width=1\linewidth]{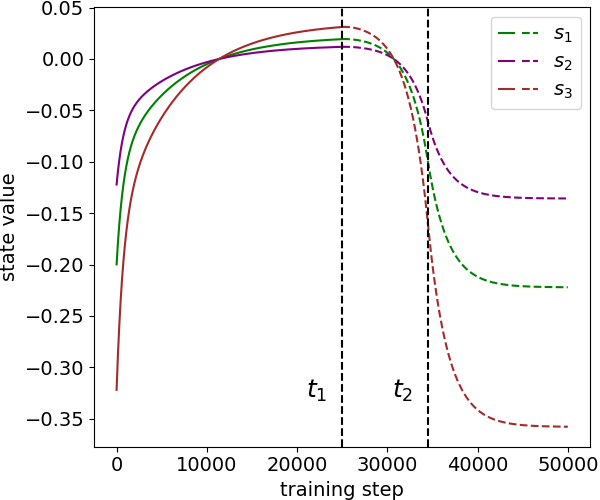}
  \caption{state value}
  \label{fig:sub-first}
\end{subfigure}
\hspace{2mm}
\begin{subfigure}[t]{.31\textwidth}
  \centering
  % include first image
  \includegraphics[width=1\linewidth]{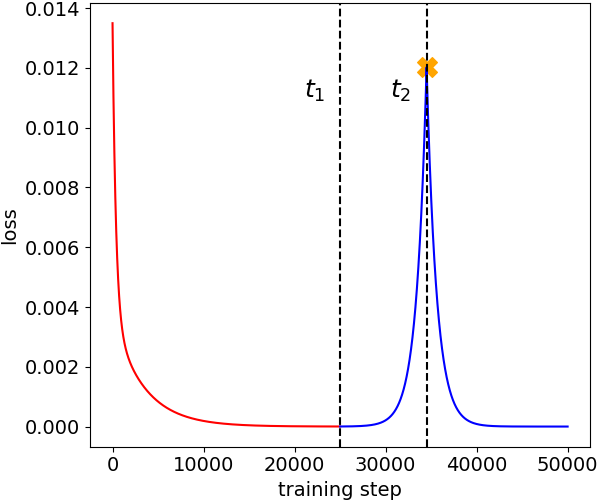}
  \caption{loss}
  \label{fig:sub-first}
\end{subfigure}
\caption{
% Training dynamics of two methods with data $\{(s_2, a_1, s_1, r)$, $(s_2,a_2,s_4,r)\}$. We first use the residual gradient descent(red) to train for $25,000$($t_1$) steps and then switch to the semi-gradient descent(blue) to train for another $25,000$ steps. In (a), the dynamics converge to $\theta_{\pi_2}$ under the residual gradient descent and then bias against $\theta_{\pi_2}$ towards $\theta_{\pi_1}$ under the semi-gradient descent. In (b), the state values change dramatically. Comparing (a) and (b), under the semi gradient descent, when the dynamics cross the policy boundary($t_2$), the loss reach the pike.
Training dynamics of two methods with the mini-batch $\{(s_1, a_1, s_2, r), (s_1, a_2, s_3, r)\}$. Initially, we employ residual gradient descent (red) for $25,000$ ($t_1$) steps, followed by a transition to semi-gradient descent (blue) for another $25,000$ steps. In (a), the dynamics converge towards $\theta_{\pi_2}$ under residual gradient descent, but exhibit a bias against $\theta_{\pi_2}$ towards $\theta_{\pi_1}$ under semi-gradient descent. In (b), the state values undergo significant changes after the transition. A comparison between (a) and (b) reveals that under semi-gradient descent, when the dynamics cross the policy boundary ($t_2$), the loss is the maxima. 
}
\label{fig::implicit_bias_true_2_semi}
\end{figure}

\textbf{Semi-gradient method bias against a saddle point in high dimension.} 
% 图\ref{}同样由数据$\{(s_1, a_1, s_2, r), (s_1, a_2, s_3, r)\}$产生，但使用了两层全连接神经网络进行拟合，神经网络的输入和线性模型一致。该神经网络宽度为$100$，在初始化时设置第一层的每个神经元的weight为$1.6$，bias为$-0.0001$，第二层的每一个神经元的weight为$0.01$，bias为$-0.0001$。在给定该初始化条件和恒定的学习率$0.002$，我们依然首先使用residual gradient descent训练10000步($t_1$时刻)，再使用semi gradient descent训练10000步。(a)表现了在训练过程中三个states的动作选择，对比(c)可以看到，当更换为semi gradient descent方法后，在loss最大的情况下($t_2$时刻)穿过了policy boundary，这一现象和图\ref{}(a)(c)展现的现象一致。在(b)中可以看到替换为semi gradient desent之后，state value也发生了明显的变化，和图\ref{}(b)中的现象保持一致。这种一致性说明residual gradient descent的收敛位置为effective loss landscape的saddle point。
We expand the experiment to a high-dimensional parameter space. Figure \ref{fig::implicit_bias_true_2_semi_dqn} is also generated by the data set $\{(s_1, a_1, s_2, r)$, $ (s_1, a_2, s_3, r)\}$, but under a DQN setting. The $Q$ is approximated by a two-layer fully connected neural network with $100$ neurons. During initialization, each neuron in the first layer was set with a weight of $1.6$ and a bias of $-0.001$, while each neuron in the second layer was set with a weight of $0.01$ and a bias of $-0.001$. Given these initial conditions and a constant learning rate of $0.002$, we first using residual gradient descent for $10,000$ steps ($t_1$), followed by semi-gradient descent for another $10,000$ steps. Figure \ref{fig::implicit_bias_true_2_semi_dqn} (a) shows the action with maximum value for three states during the training process. In comparison with Figure \ref{fig::implicit_bias_true_2_semi_dqn} (c), the loss reach the maxima when cross the policy boundary($t_2$), which is similar with Figure \ref{fig::implicit_bias_true_2_semi}. In addition, the dynamics of state values in Figure \ref{fig::implicit_bias_true_2_semi_dqn} (b) is also similar with Figure \ref{fig::implicit_bias_true_2_semi} (b). These similarities suggest that the convergence position of the residual gradient descent is a saddle point on the effective loss landscape.

\begin{figure}[htb]
\centering
\begin{subfigure}[t]{.31\textwidth}
  \centering
  % include first image
  \includegraphics[width=1\linewidth]{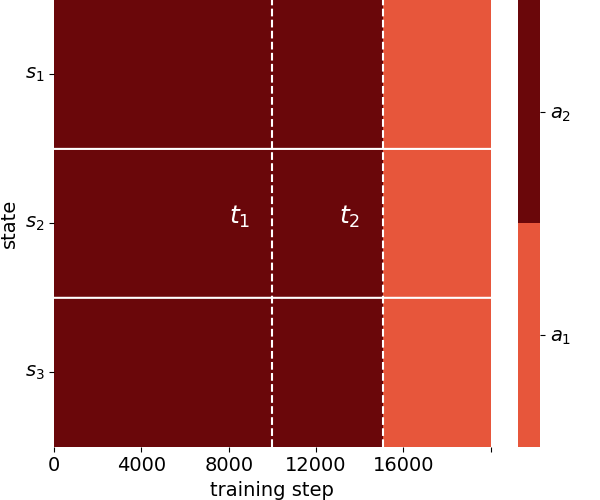}
  \caption{policy}
  \label{fig:sub-first}
\end{subfigure}
\hspace{2mm}
\begin{subfigure}[t]{.31\textwidth}
  \centering
  % include first image
  \includegraphics[width=1\linewidth]{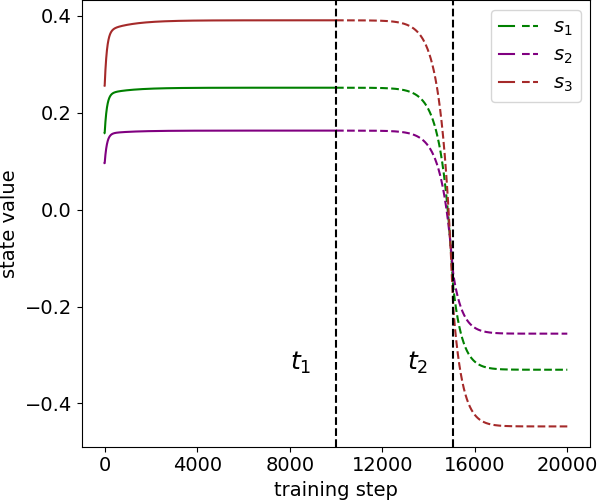}
  \caption{state value}
  \label{fig:sub-first}
\end{subfigure}
\hspace{2mm}
\begin{subfigure}[t]{.31\textwidth}
  \centering
  % include first image
  \includegraphics[width=1\linewidth]{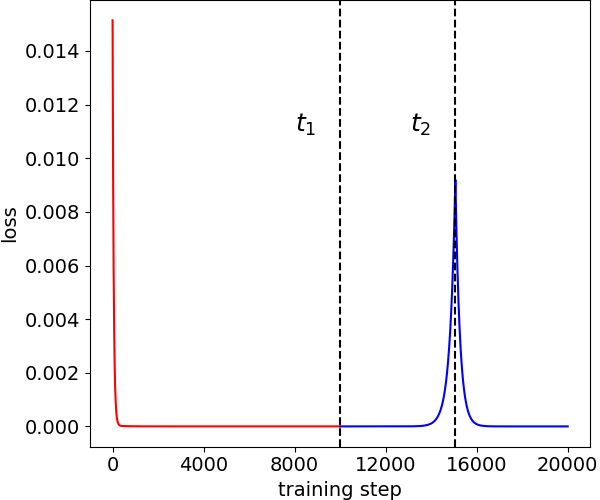}
  \caption{loss}
  \label{fig:sub-first}
\end{subfigure}
\caption{
% Training dynamics of two methods with data $\{(s_2, a_1, s_1, r)$, $(s_2,a_2,s_4,r)\}$ and the neural network. We first use the residual gradient descent for $10,000$ ($t_1$) steps and switch to the semi-gradient descent for $15,000$ steps. (a) shows the action with the maximum state value for each state during training. Compare (a) and (c), the dynamics goes cross the policy boundary when the loss reach the maximum ($t_2$). (b) demonstrate a significant change of state value after switch to the semi-gradient descent. From the phenomena mentioned above, we know the residual gradient descent converges to a saddle point for effective loss landscape.
The training dynamics of two methods with the mini-batch $\{(s_1, a_1, s_2, r), (s_1, a_2, s_3, r)\}$ involve a neural network. Initially, we employ residual gradient descent for $10,000$ ($t_1$) steps and then transition to semi-gradient descent for another $10,000$ steps. In (a), the action with maximum state value for each state during training is demonstrated. A comparison between (a) and (c) reveals that the dynamics cross the policy boundary when the loss reaches its peak ($t_2$). This implies that the saddle point and the global minima that semi-gradient converges are situated on opposite sides of the policy boundary. (b) illustrates a significant change in state values after switching to semi-gradient descent. The observations above indicate that the residual gradient descent converges to a saddle point, and this saddle point, along with the global minima to which the semi-gradient converges, is located on opposite sides of the policy boundary.
}
\label{fig::implicit_bias_true_2_semi_dqn}
\end{figure}

\section{Implicit Bias of the Semi-gradient Method with More Realistic Data}
\label{sec::high_dimension}

% 给出一个grid world环境如图\ref{}所示，状态空间$S$共有$19$个状态，其中$s_16$为terminal state。动作空间$A=\{a_1,a_2,a_3,a_4\}$，分别表示“上，下，左，右”。当状态处于边界时，跨越边界的动作会反弹会当前状态,i.e., $f(s_1,a_1)=s_1$。奖励函数设置为：当下一个状态为$s_{12}$, $s_{13}$和$s_{14}$时，奖励为$-1$，当下一个状态为$s_15$时，奖励为$+1$。设定discount factor为$\gamma =0.98$。
\begin{exam}[a grid world environment]
    \label{exam::grid_world}
    Given a grid world environment as illustrated in Figure \ref{fig::grid_world_example}, the state space $S$ consists of $19$ states, with $s_{16}$ being the terminal state. The action space $A = \{a_1, a_2, a_3, a_4\}$ corresponds to the directions "up, down, left, right," respectively. When a state is at the boundary, an action that would cross this boundary results in a bounce-back to the current state, i.e., $f(s_1, a_1) = s_1$. The reward function is defined as follows: a reward of $-1$ is given when the next state is $s_{12}$, $s_{13}$, or $s_{14}$, and a reward of $+1$ is granted when the next state is $s_{15}$. The discount factor is set to $\gamma = 0.98$.
\end{exam}

\begin{figure}[htb]
\centering
% include first image
\includegraphics[width=0.25\linewidth]{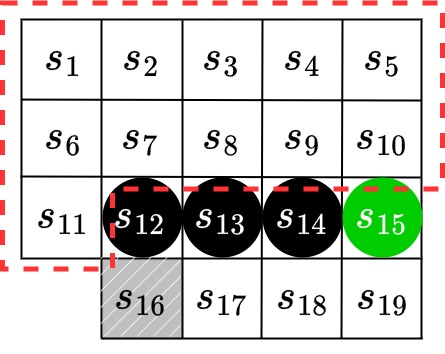}
\caption{A grid world environment. }
\label{fig::grid_world_example}
\end{figure}

\textbf{Data sampling strategy and neural network settings.}
% 我们选择图\ref{}中红色框中的所有状态和整个动作空间的笛卡尔积作为训练数据。这15个状态的embedding为随机采样的$15 \times 15$的矩阵，并且在采样后对每一行做归一化处理，其中随机种子为four。另外，我们使用四层全连接网络拟合$Q$，其中第一个隐藏层的宽度为256，第二个隐藏层层的宽度为512，第三个隐藏层的宽度为1024。
In order to demonstrate the existence of saddle point in a larger sample dataset, we consider a mini-batch contains the Cartesian product of all states within the red box in Figure \ref{fig::grid_world_example} and the entire action space, which has totally 44 sample data. The embeddings for these 15 states are represented by a randomly sampled $15 \times 15$ matrix, and after sampling, each row is normalized. The random seed used for sampling is $4$. Furthermore, we employ a four-layer fully connected network to approximate the $Q$ function, where the first hidden layer has a width of 512, the second hidden layer has a width of 1024, and the third hidden layer has a width of 1024. The model is initialized with random seed $75$. Additionally, we use an NVIDIA GeForce RTX 3080 GPU to train the model.

\textbf{Existence of saddle point in high dimension with more realistic data.}
% 基于上述的设定，我们首先使用residual gradient descent训练10000步($t_1$时刻)，此时设定learning rate为0.3，momentum为0.8，dumped为0.1。之后使用semi gradient descent训练15000步，learning rate为0.1。最终我们可以得到如图\ref{}的training dynamics。从(b)可以看到，在替换为semi gradient descent之后state value发生了明显的变化，这与图\ref{}(b)和图\ref{}(b)展示的现象一致。在(a)中，我们仅选取了16000步到17000步进行policy dynamics的可视化(整体的图片可以在Appendix上看到)，其中$t_2$时刻为(c)中误差上升到的最高点，可以看到$t_2$时刻附近training dynamics跨越了一个policy boundary，这同样和图\ref{}(a)和图\ref{}(a)展示的现象一致。这种一致性，表明了residual gradient descent收敛的点为effective loss landscape上的saddle point。
% Based on the above settings, we initially trained the model using residual gradient descent for 10,000 ($t_1$) steps, with a learning rate of 0.3, momentum of 0.8, and damping of 0.1. Afterwords, we trained the model using semi-gradient descent for 15,000 steps with a learning rate of 0.1. In Figure \ref{fig::dynamics_grid_world} (a), we selected only steps 16,000 to 17,000 for the visualization of policy dynamics (the overall image can be viewed in the Appendix \ref{}), where $t_2$ corresponds to the peak of the error in (c). It is evident that around the moment $t_2$, the training dynamics crossed a policy boundary, which aligns with the phenomena shown in Figure \ref{fig::implicit_bias_true_2_semi} (a). The state values in (b) also changes dramatically. The phenomena mentioned above indicates the residual gradient descent converges to a saddle point in effective loss landscape.
Given the above settings, we initially trained the model using residual gradient descent for 10,000 ($t_1$) steps, employing a learning rate of 0.3, momentum of 0.8, and damping of 0.1. Subsequently, the model was trained using semi-gradient descent for 15,000 steps with a learning rate of 0.1. In Figure \ref{fig::dynamics_grid_world} (a), we demonstrated the action with maximum value, specifically focused on steps 16,000 to 17,000 (the complete image is available in Figure \ref{fig::policy_dynamics_full}), where $t_2$ corresponds to the peak error in (c). Notably, around the time of $t_2$, the training dynamics crossed a policy boundary, consistent with the observations in Figure \ref{fig::implicit_bias_true_2_semi} (a). In addition, there is a significant shift in state values as depicted in (b). These observations suggest that residual gradient descent converges to a saddle point in the effective loss landscape.

\begin{figure}[htb]
\centering
\begin{subfigure}[t]{.31\textwidth}
  \centering
  % include first image
  \includegraphics[width=1\linewidth]{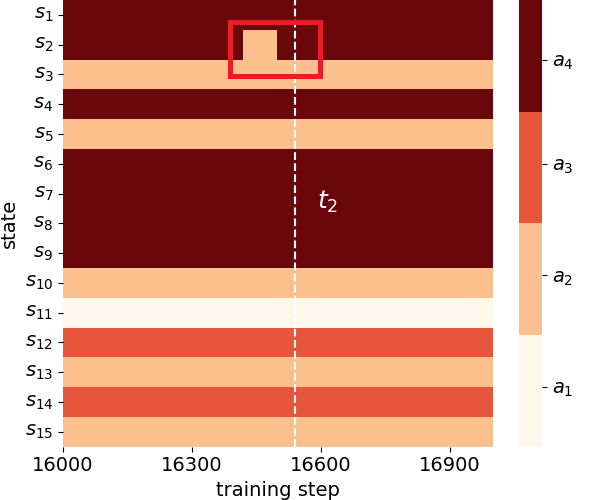}
  \caption{policy}
  \label{fig:sub-first}
\end{subfigure}
\hspace{2mm}
\begin{subfigure}[t]{.31\textwidth}
  \centering
  % include first image
  \includegraphics[width=1\linewidth]{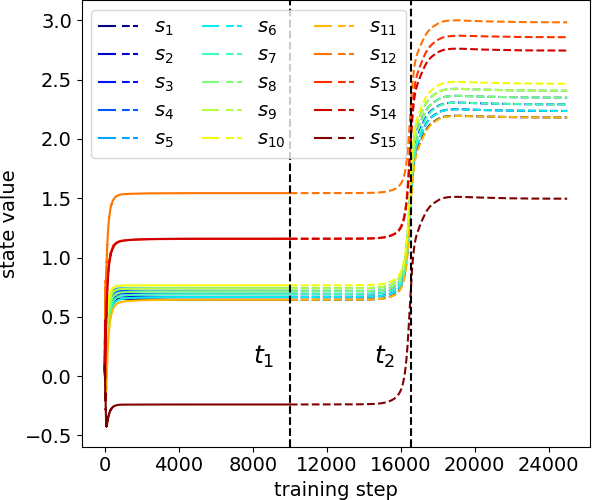}
  \caption{state value}
  \label{fig:sub-first}
\end{subfigure}
\hspace{2mm}
\begin{subfigure}[t]{.31\textwidth}
  \centering
  % include first image
  \includegraphics[width=1\linewidth]{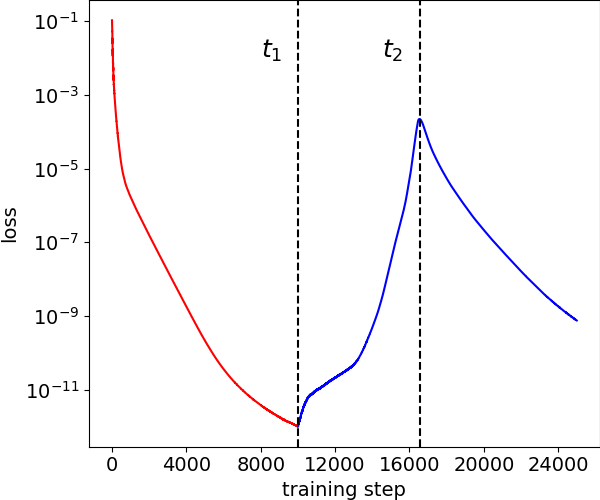}
  \caption{loss}
  \label{fig:sub-first}
\end{subfigure}
\caption{
% Training dynamics of two methods with Example \ref{exam::grid_world} and neural network. Initially, we employ residual gradient descent for $10,000$ ($t_1$) steps and then transition to semi-gradient descent for $15,000$ steps. In Figure \ref{fig::dynamics_grid_world} (a), we demonstrated the action with maximum value, specifically focused on steps 16,000 to 17,000 (the complete image is available in Appendix \ref{}), where $t_2$ corresponds to the peak error in (c). Notably, around the time of $t_2$, the training dynamics crossed a policy boundary, consistent with the observations in Figure \ref{fig::implicit_bias_true_2_semi} (a). In addition, there is a significant shift in state values as depicted in (b). These observations suggest that residual gradient descent converges to a saddle point in the effective loss landscape.
The training dynamics of two methods using Example \ref{exam::grid_world} and a neural network were compared. Initially, we used residual gradient descent for $10,000$ ($t_1$) steps and then switched to semi-gradient descent for $15,000$ steps. In (a), we highlighted the action with the maximum value, focusing on steps 16,000 to 17,000. Around $t_2$, which corresponds to the peak error in (c), the training dynamics crossed a policy boundary, aligning with findings in Figure \ref{fig::implicit_bias_true_2_semi} (a). Notably, there was a significant shift in state values, as shown in (b). The observations above indicate that the residual gradient descent converges to a saddle point, and this saddle point, along with the global minima to which the semi-gradient converges, is located on opposite sides of the policy boundary.
}
\label{fig::dynamics_grid_world}
\end{figure}

\section{Analyze the Transition of Global Minimum to Saddle Point in $\mathbb{R}^2$}
\label{sec::theorem}

\begin{assum}
    \label{assum:mdpAssum}
    Given a MDP with $A=\{a_1,a_2\}$ and sample two data from it. Assume the following three statements hold: 
    \begin{enumerate}[label = (\arabic*), itemindent = 0pt, labelindent = \parindent, labelwidth = 2em, labelsep = 5pt, leftmargin = *]
        \item The sample data is $\{(s_{\alpha},a_1,s'_{\alpha},C)$, $(s_{\beta},a_2,s'_{\beta},C)\}$ with $C < 0$. 
        
        \item $Q(s,a)=\phi(s)\theta(a)$, where $\theta = [\theta(a_1), \theta(a_2)]^{\rm{T}}$, $\theta(a_1), \theta(a_2) \in \mathbb{R}$. 
        
        \item $\phi(s_{\alpha}), \phi(s_{\beta}) > 0$, $\phi(s'_{\alpha}),\phi(s'_{\beta}) \geq 0$ and $\phi(s_{\alpha}) - \gamma \phi(s'_{\alpha}) \neq 0$, $\phi(s_{\beta})-\gamma \phi(s'_{\beta}) \neq 0$. 
    \end{enumerate}
\end{assum}

Under Assumption \ref{assum:mdpAssum}, the negative residual gradient and semi-gradient, which is the "force" in Fokker--Planck equation, is defined as follow. We define the policy for states $s'_{\alpha}$, $s'_{\beta}$ as $\pi_1(s'_{\alpha})=a_1$, $\pi_1(s'_{\beta})=a_1$ and $\pi_2(s'_{\alpha})=a_2$, $\pi_2(s'_{\beta})=a_2$. The negative residual gradient with $\pi_1$ is defined as
\begin{equation}
    \label{eq::trueGradientPi1}
    F^{\pi_1}_{\rm{res}}(\theta) = \left (
    \begin{aligned}
        & -(\phi(s_{\alpha}) - \gamma \phi(s'_{\alpha}))[\phi(s_{\alpha})\theta(a_1) - C - \gamma \phi(s'_{\alpha})\theta(a_1)] \\  
        & ~~~~~ + \gamma \phi(s'_{\beta})[\phi(s_{\beta}) \theta(a_2) - C - \gamma \phi(s'_{\beta}) \theta(a_1)]], \\
        & -\phi(s_{\alpha})[\phi(s_{\alpha})\theta(a_2) - C - \gamma \phi(s'_{\alpha})\theta(a_1)]
    \end{aligned}
    \right).
\end{equation}
The negative residual gradient with $\pi_2$ is defined as
\begin{equation}
    \label{eq::trueGradientPi2}
    F^{\pi_2}_{\rm{res}}(\theta) = \left(
    \begin{aligned}
        & -\phi(s_{\alpha})[\phi(s_{\alpha})\theta(a_1) - C - \gamma \phi(s'_{\alpha})\theta(a_2)], \\
        & \gamma\phi(s'_{\alpha})[\phi(s_{\alpha})\theta(a_1) - C - \gamma \phi(s'_{\alpha})\theta(a_2)] \\ 
        & ~~~~~ - (\phi(s_{\beta}) - \gamma \phi(s'_{\beta}))[\phi(s_{\beta})\theta(a_2) - C - \gamma \phi(s'_{\beta})\theta(a_2)]
    \end{aligned}
    \right).
\end{equation}
The negative semi-gradient with $\pi_1$ is defined as
\begin{equation}
\label{eq::semiGradientPi1}
    F^{\pi_1}_{\rm{semi}}(\theta) = \left(
    \begin{aligned}
        & - \phi(s_{\alpha})[\phi(s_{\alpha})\theta(a_1) - C - \gamma \phi(s'_{\alpha})\theta(a_1)], \\
       & -\phi(s_{\beta})[\phi(s_{\beta})\theta(a_2) - C - \gamma \phi(s'_{\beta})\theta(a_1)] 
    \end{aligned}
    \right).
\end{equation}
The negative semi-gradient with $\pi_2$ is defined as
\begin{equation}
\label{eq::semiGradientPi2}
    F^{\pi_2}_{\rm{semi}}(\theta) = \left (
    \begin{aligned}
        & - \phi(s_{\alpha})[\phi(s_{\alpha})\theta(a_1) - C - \gamma \phi(s'_{\alpha})\theta(a_2)], \\
        & -\phi(s_{\beta})[\phi(s_{\beta})\theta(a_2) - C - \gamma \phi(s'_{\beta})\theta(a_2)] 
    \end{aligned}
    \right).
\end{equation}
We define $F_{\rm{semi}}(\theta)$ and $F_{\rm{res}}(\theta)$ as the uniform representation of the force vector. Specifically, $F_{\rm{semi}}(\theta) = F^{\pi_1}_{\rm{semi}}(\theta)$ and $F_{\rm{res}}(\theta) = F^{\pi_1}_{\rm{res}}(\theta)$ when $\theta(a_1) \geq \theta(a_2)$, and $F_{\rm{semi}}(\theta) = F^{\pi_2}_{\rm{semi}}(\theta)$ and $F_{\rm{res}}(\theta) = F^{\pi_2}_{\rm{res}}(\theta)$ when $\theta(a_1) \leq \theta(a_2)$.

\begin{lem}[existence of solution]
    \label{lem::existenceOfGloablMin}
    Suppose Assumption \ref{assum:mdpAssum} holds, the solution for the Bellman optimal loss with policy $\pi_1$ is $\theta_{\pi_1}(a_1) = \frac{C}{\phi(s_{\alpha}) - \gamma \phi(s'_{\alpha})}$ and $\theta_{\pi_1}(a_2) = \frac{C}{\phi(s_{\beta})} + \frac{C\gamma \phi(s'_{\beta})}{\phi(s_{\beta}) \Bigl( \phi(s_{\alpha}) - \gamma \phi(s'_{\alpha}) \Bigr)}$, it exists when $\frac{\phi(s_{\beta})-\gamma \phi(s'_{\beta}) }{\phi(s_{\alpha})-\gamma \phi(s'_{\alpha})} \leq 1$. The solution with policy $\pi_2$ is $\theta_{\pi_2}(a_1) = \frac{C}{\phi(s_{\alpha})} +  \frac{C\gamma \phi(s'_{\alpha})}{\phi(s_{\alpha}) \Bigl( \phi(s_{\beta}) - \gamma \phi(s'_{\beta}) \Bigr)} $ and $\theta_{\pi_2}(a_2) = \frac{C}{\phi(s_{\beta}) - \gamma \phi(s'_{\beta})}$, it exists when $\frac{\phi(s_{\alpha})-\gamma \phi(s'_{\alpha})}{ \phi(s_{\beta})-\gamma \phi(s'_{\beta})} \leq 1$.
\end{lem}

\begin{remark}
    % Notice that if $\frac{\phi(s_{\alpha})-\gamma \phi(s'_{\alpha})}{ \phi(s_{\beta})-\gamma \phi(s'_{\beta})} = \frac{\phi(s_{\beta})-\gamma \phi(s'_{\beta}) }{\phi(s_{\alpha})-\gamma \phi(s'_{\alpha})} = 1$, the two solutions are the same and located on the policy boundary. The two solutions exist and distinct if and only if the following conditions hold: 1) $\phi(s_{\alpha})-\gamma \phi(s'_{\alpha}) > 0$, $\phi(s_{\beta})-\gamma \phi(s'_{\beta}) < 0$ or 2) $\phi(s_{\alpha})-\gamma \phi(s'_{\alpha}) < 0$, $\phi(s_{\beta})-\gamma \phi(s'_{\beta}) > 0$.
    If $\frac{\phi(s_{\alpha})-\gamma \phi(s'_{\alpha})}{ \phi(s_{\beta})-\gamma \phi(s'_{\beta})} = \frac{\phi(s_{\beta})-\gamma \phi(s'_{\beta}) }{\phi(s_{\alpha})-\gamma \phi(s'_{\alpha})} = 1$, the two solutions coincide and lie on the policy boundary. The two solutions are distinct and exist only if the following conditions are met: 1) $\phi(s_{\alpha})-\gamma \phi(s'_{\alpha}) > 0$, $\phi(s_{\beta})-\gamma \phi(s'_{\beta}) < 0$; or 2) $\phi(s_{\alpha})-\gamma \phi(s'_{\alpha}) < 0$, $\phi(s_{\beta})-\gamma \phi(s'_{\beta}) > 0$.
\end{remark}

\begin{lem}[smoothness for effective loss landscape]
    \label{lem::continuitySemi}
    Suppose Assumption \ref{assum:mdpAssum} holds, for the semi-gradient in all $\theta \in \mathbb{R}^2$ we have $\lim_{\norm{h} \to 0} \left(  F_{\rm{semi}}(\theta + h) - F_{\rm{semi}}(\theta - h)  \right)  = 0$. Besides, for the residual gradient, there exist $\theta \in \mathbb{R}^2$ such that $\lim_{\norm{h} \to 0} \left (  F_{\rm{res}}(\theta + h) - F_{\rm{res}}(\theta - h)  \right) \neq 0$.
\end{lem}

\begin{remark}
    Lemma \ref{lem::continuitySemi} prove the continuity of semi-gradient in $\mathbb{R}^2$, this directly indicates the smoothness of the effective loss function.
\end{remark}

% 再加一个True Gradient的定理，再加一个remark，说明对比。

\begin{thm}[implicit bias of semi-gradient]
    \label{thm::SaddPointRange}
    Suppose Assumption \ref{assum:mdpAssum} holds and $\phi(s_{\alpha})=\phi(s_{\beta})$, given a set $\Omega=\{ \theta \in \mathbb{R}^2 | \theta = \lambda \theta_{\pi_1} + (1-\lambda) \theta_{\pi_2}, \lambda \in (0,1) \}$, the following two statements hold:
    \begin{enumerate}[label = (\arabic*), itemindent = 0pt, labelindent = \parindent, labelwidth = 2em, labelsep = 5pt, leftmargin = *]
        \item if $\phi(s_{\alpha}) - \gamma \phi(s'_{\alpha}) > 0$ and $\phi(s_{\beta}) - \gamma \phi(s'_{\beta}) < 0$, then for all $\theta \in \Omega \cup \{ \theta(a_2) \geq \theta(a_1) \}$, $\frac{\langle (\theta_{\pi_1} - \theta_{\pi_2}), F^{\pi_2}_{\rm{semi}}(\theta_{\eta}) \rangle}{\norm{\theta_{\pi_1} - \theta_{\pi_2}} \norm{F^{\pi_2}_{\rm{semi}}(\theta_{\eta})}}=1$ and for all $\theta' \in \Omega \cup \{ \theta(a_2) \leq \theta(a_1) \}$, $\frac{\langle (\theta_{\pi_1} - \theta_{\pi_2}), F^{\pi_1}_{\rm{semi}}(\theta'_{\eta}) \rangle}{\norm{\theta_{\pi_1} - \theta_{\pi_2}} \norm{F^{\pi_1}_{\rm{semi}}(\theta'_{\eta})}} = 1$.
        
        \item if $\phi(s_{\alpha}) - \gamma \phi(s'_{\alpha}) < 0$ and $\phi(s_{\beta}) - \gamma \phi(s'_{\beta}) > 0$, then for all $\theta \in \Omega \cup \{ \theta(a_2) \geq \theta(a_1) \}$, $\frac{\langle (\theta_{\pi_1} - \theta_{\pi_2}), F^{\pi_2}_{\rm{semi}}(\theta_{\eta}) \rangle}{\norm{\theta_{\pi_1} - \theta_{\pi_2}} \norm{F^{\pi_2}_{\rm{semi}}(\theta_{\eta})}}=-1$ and for all $\theta' \in \Omega \cup \{ \theta(a_2) \leq \theta(a_1) \}$, $\frac{\langle (\theta_{\pi_1} - \theta_{\pi_2}), F^{\pi_1}_{\rm{semi}}(\theta'_{\eta}) \rangle}{\norm{\theta_{\pi_1} - \theta_{\pi_2}} \norm{F^{\pi_1}_{\rm{semi}}(\theta'_{\eta})}} = -1$.
    \end{enumerate}
    $\langle \cdot \rangle$ donated as the inner product.
\end{thm}

\begin{remark}
    % Theorem \ref{thm::SaddPointRange} provides a theoretical understanding of the implicit bias of semi gradient demonstrated in Figure \ref{fig::implicit_bias_true_2_semi}. This statement tells us: if the condition in (1) is satisfied and a training dynamics falls in the line between $\theta_{\pi_1}$ and $\theta_{\pi_2}$, the training dynamics will follow the line and reach $\theta_{\pi_1}$. On the contrary, if the condition (2) is satisfied, the training dynamics will follow the line and reach $\theta_{\pi_2}$. 
    Theorem \ref{thm::SaddPointRange} offers a theoretical explanation of the implicit bias of semi-gradient as illustrated in Figure \ref{fig::implicit_bias_true_2_semi}. This theorem states that if condition (1) is met and the training dynamics lie on the line between $\theta_{\pi_1}$ and $\theta_{\pi_2}$, they will converge along the line towards $\theta_{\pi_1}$. Conversely, if condition (2) is met, the training dynamics will converge along the line towards $\theta_{\pi_2}$.
\end{remark}

\begin{remark}
    % Theorem \ref{thm::SaddPointRange}也意味着如果(1)的条件满足，则$\theta_{\pi_2}$是一个saddle；如果（2）的条件满足，则$\theta_{\pi_1}$是一个saddle。这是因为$\theta_{\pi_1}$和$\theta_{\pi_2}$都是critical point，而critical point通常分为两类，global/local minimum以及saddle，并且只有saddle才存在一条连续的轨迹从critical point附近出发并逐渐远离该critical point。Theorem \ref{thm::SaddPointRange}给出了这条轨迹的存在性，因此说明了saddle的存在性。
    Theorem \ref{thm::SaddPointRange} also suggests that if condition (1) holds, then $\theta_{\pi_2}$ is a saddle point; if condition (2) holds, then $\theta_{\pi_1}$ is a saddle point. This is due to the fact that both $\theta_{\pi_1}$ and $\theta_{\pi_2}$ are critical points—global/local minima or saddle points—in the effective loss landscape. Only saddle points exhibit a trajectory moving away from it. The existence of such trajectories is guaranteed by Theorem \ref{thm::SaddPointRange}, thereby confirming the presence of saddle points.
\end{remark}

\section{Conclusion and Discussion}
\label{sec::conclusion}
% 在本文中，我们主要讨论了semi-gradient Q-learning的implict bias。具体来说，我们首先通过Wang's potential landscape theory构建了二维参数空间下的effective loss landscape，并进行可视化。我们通过可视化发现global minima in loss landscape can transit into saddle point in effective loss landscape，这导致了semi-gradient method会bias against the convergence point found by the residual-gradient method。另外，在两维参数空间下在effective loss landscape附近开始的training dynamics可能会发散。之后，我们展示了在使用神经网络拟合状态值函数$Q$时，在effective loss landscape同样存在saddle point。本文为理解semi-gradient Q-learning在参数空间中的implicit bias提供了新的视角。
In this work, we primarily discuss the implicit bias of semi-gradient Q-learning. Specifically, we first constructed and visualized the effective loss landscape within a two-dimensional parameter space, based on Wang's potential landscape theory. Through visualization, we discovered that global minima in the loss landscape can transition into saddle points in the effective loss landscape. This transition makes the semi-gradient method bias against the convergence point found by the residual-gradient method. Subsequently, we demonstrated that a global minima on the loss landscape can also transit to a saddle point on the effective loss landscape when using neural networks to approximate the state-action value function $Q$. This paper provides a new approach for understanding the implicit bias of semi-gradient Q-learning within the parameter space.

% Limitation。本文主要由两个limitation：首先，本文仅在两维参数空间中展示了semi-gradient Q-learning的发散情况，但是没有在高维情况下找出对应的发散情况。其次，本文仅在二维参数空间给出了semi-gradient Q-learning的出现saddle point的条件和对应的implicit bias的理论理解，但是在高维空间中semi-gradient的implicit bias缺乏理论理解。
\textbf{Limitation:} This work only provides a theoretical understanding of the implicit bias in two-dimensional parameter space, but the theoretical understanding in higher-dimensional spaces is lacking. We are going to provide a theoretical understanding of the implicit bias of the semi-gradient Q-learning with neural networks in future works.

%Bibliography
\bibliographystyle{unsrt}  
\bibliography{references}  

\begin{thebibliography}{10}

\bibitem{mnih2015human}
Volodymyr Mnih, Koray Kavukcuoglu, David Silver, Andrei~A Rusu, Joel Veness, Marc~G Bellemare, Alex Graves, Martin Riedmiller, Andreas~K Fidjeland, Georg Ostrovski, et~al.
\newblock Human-level control through deep reinforcement learning.
\newblock {\em Nature}, 518(7540):529--533, 2015.

\bibitem{silver2017mastering}
David Silver, Julian Schrittwieser, Karen Simonyan, Ioannis Antonoglou, Aja Huang, Arthur Guez, Thomas Hubert, Lucas Baker, Matthew Lai, Adrian Bolton, et~al.
\newblock Mastering the game of go without human knowledge.
\newblock {\em Nature}, 550(7676):354--359, 2017.

\bibitem{vinyals2019grandmaster}
Oriol Vinyals, Igor Babuschkin, Wojciech~M Czarnecki, Micha{\"e}l Mathieu, Andrew Dudzik, Junyoung Chung, David~H Choi, Richard Powell, Timo Ewalds, Petko Georgiev, et~al.
\newblock Grandmaster level in starcraft ii using multi-agent reinforcement learning.
\newblock {\em Nature}, 575(7782):350--354, 2019.

\bibitem{deng2021unified}
Yang Deng, Yaliang Li, Fei Sun, Bolin Ding, and Wai Lam.
\newblock Unified conversational recommendation policy learning via graph-based reinforcement learning.
\newblock In {\em Proceedings of the 44th International ACM SIGIR Conference on Research and Development in Information Retrieval}, pages 1431--1441, 2021.

\bibitem{bello2016neural}
Irwan Bello, Hieu Pham, Quoc~V Le, Mohammad Norouzi, and Samy Bengio.
\newblock Neural combinatorial optimization with reinforcement learning.
\newblock {\em arXiv preprint arXiv:1611.09940}, 2016.

\bibitem{khalil2017learning}
Elias Khalil, Hanjun Dai, Yuyu Zhang, Bistra Dilkina, and Le~Song.
\newblock Learning combinatorial optimization algorithms over graphs.
\newblock {\em Advances in Neural Information Processing Systems}, 30, 2017.

\bibitem{sutton2018reinforcement}
Richard~S Sutton and Andrew~G Barto.
\newblock {\em Reinforcement learning: An introduction}.
\newblock MIT press, 2018.

\bibitem{tsitsiklis1996analysis}
John~N Tsitsiklis and Benjamin~Van Roy.
\newblock An analysis of temporal-difference learning with function approximationtechnical.
\newblock {\em Rep. LIDS-P-2322). Lab. Inf. Decis. Syst. Massachusetts Inst. Technol. Tech. Rep}, 1996.

\bibitem{van2018deep}
Hado Van~Hasselt, Yotam Doron, Florian Strub, Matteo Hessel, Nicolas Sonnerat, and Joseph Modayil.
\newblock Deep reinforcement learning and the deadly triad.
\newblock {\em arXiv preprint arXiv:1812.02648}, 2018.

\bibitem{achiam2019towards}
Joshua Achiam, Ethan Knight, and Pieter Abbeel.
\newblock Towards characterizing divergence in deep q-learning.
\newblock {\em arXiv preprint arXiv:1903.08894}, 2019.

\bibitem{baird1995residual}
Leemon Baird.
\newblock Residual algorithms: Reinforcement learning with function approximation.
\newblock In {\em Machine Learning Proceedings 1995}, pages 30--37. Elsevier, 1995.

\bibitem{saleh2019deterministic}
Ehsan Saleh and Nan Jiang.
\newblock Deterministic bellman residual minimization.
\newblock In {\em Proceedings of Optimization Foundations for Reinforcement Learning Workshop at NeurIPS}, 2019.

\bibitem{zhang2019deep}
Shangtong Zhang, Wendelin Boehmer, and Shimon Whiteson.
\newblock Deep residual reinforcement learning.
\newblock {\em arXiv preprint arXiv:1905.01072}, 2019.

\bibitem{vardi2023implicit}
Gal Vardi.
\newblock On the implicit bias in deep-learning algorithms.
\newblock {\em Communications of the ACM}, 66(6):86--93, 2023.

\bibitem{belkin2019reconciling}
Mikhail Belkin, Daniel Hsu, Siyuan Ma, and Soumik Mandal.
\newblock Reconciling modern machine-learning practice and the classical bias--variance trade-off.
\newblock {\em Proceedings of the National Academy of Sciences}, 116(32):15849--15854, 2019.

\bibitem{ergen2021convex}
Tolga Ergen and Mert Pilanci.
\newblock Convex geometry and duality of over-parameterized neural networks.
\newblock {\em Journal of Machine Learning Research}, 22(212):1--63, 2021.

\bibitem{keskar2017large}
Nitish~Shirish Keskar, Jorge Nocedal, Ping Tak~Peter Tang, Dheevatsa Mudigere, and Mikhail Smelyanskiy.
\newblock On large-batch training for deep learning: Generalization gap and sharp minima.
\newblock In {\em 5th International Conference on Learning Representations, ICLR 2017}, 2017.

\bibitem{wang2008potential}
Jin Wang, Li~Xu, and Erkang Wang.
\newblock Potential landscape and flux framework of nonequilibrium networks: Robustness, dissipation, and coherence of biochemical oscillations.
\newblock {\em Proceedings of the National Academy of Sciences}, 105(34):12271--12276, 2008.

\bibitem{zhou2016construction}
Peijie Zhou and Tiejun Li.
\newblock Construction of the landscape for multi-stable systems: Potential landscape, quasi-potential, a-type integral and beyond.
\newblock {\em The Journal of Chemical Physics}, 144(9), 2016.

\bibitem{schoknecht2003td}
Ralf Schoknecht and Artur Merke.
\newblock Td (0) converges provably faster than the residual gradient algorithm.
\newblock In {\em Proceedings of the 20th International Conference on Machine Learning (ICML-03)}, pages 680--687, 2003.

\bibitem{li2008worst}
Lihong Li.
\newblock A worst-case comparison between temporal difference and residual gradient with linear function approximation.
\newblock In {\em Proceedings of the 25th International Conference on Machine Learning}, pages 560--567, 2008.

\bibitem{lillicrap2015continuous}
Timothy~P Lillicrap, Jonathan~J Hunt, Alexander Pritzel, Nicolas Heess, Tom Erez, Yuval Tassa, David Silver, and Daan Wierstra.
\newblock Continuous control with deep reinforcement learning.
\newblock {\em arXiv preprint arXiv:1509.02971}, 2015.

\bibitem{neyshabur2014search}
Behnam Neyshabur, Ryota Tomioka, and Nathan Srebro.
\newblock In search of the real inductive bias: On the role of implicit regularization in deep learning.
\newblock {\em ICLR workshop 2015}, 2015.

\bibitem{holubec2019physically}
Viktor Holubec, Klaus Kroy, and Stefano Steffenoni.
\newblock Physically consistent numerical solver for time-dependent fokker-planck equations.
\newblock {\em Physical Review E}, 99(3):032117, 2019.

\end{thebibliography}

\newpage

\begin{appendix}

\section{Wang's Potential Landscape Theory}
\label{sec::fpeInfo}
In this section, we introduce Fokker Planck equation and Wang's potential landscape theory as a preparation for constructing the effective loss landscape for the semi-gradient method. Let's consider a Fokker--Planck equation with two variables, 
\begin{equation}
    \frac{ \partial \rho(x_1, x_2,t)}{\partial t} = - \sum_{i=1}^2 \frac{\partial }{\partial x_i} [F_i(x_1,x_2,t) \rho(x_1,x_2,t)] + \sum_{i=1}^2 \sum_{j=1}^2 \frac{\partial^2}{\partial x_i \partial x_j} [D_{ij}(x_1,x_2,t) \rho(x_1,x_2,t)].
\end{equation}
$\rho$ is a distribution, $F$ is the force vector or drift term, $D$ is the diffusion matrix. We assume $D = \sigma I$, where $I$ is the identity matrix and $\sigma$ is the diffusion constant, and the force is time-independent $F(x_1,x_2,t) = F(x_1,x_2)$. Then the Fokker--Planck equation can be reduce to
\begin{align}
     \frac{ \partial \rho(x_1, x_2,t)}{\partial t} & = - \sum_{i=1}^2 \frac{\partial }{\partial x_i} \left\{F_i(x_1,x_2) \rho(x_1,x_2,t) - \sigma \frac{\partial\rho(x_1,x_2,t)}{\partial x_i}\right\} \\
     & = -\nabla \cdot \{F(x_1,x_2) \rho(x_1,x_2,t) - \sigma \nabla \rho(x_1,x_2,t)\} = -\nabla \cdot J(x_1,x_2,t).
\end{align}
$J(x_1,x_2,t)$ is the probability flux vector. We define $\rho_{ss}$ as a stationary distribution, $J_{ss}(x_1,x_2)$ as the flux corresponding to it, and we have 
\begin{equation}
    \frac{\partial \rho_{ss}(x_1,x_2)}{\partial t} = 0 \Rightarrow \nabla \cdot J_{ss}(x_1,x_2) = 0.
\end{equation}
There are two different values for the flux to reach the stationary distribution. The first case is $J_{ss}(x_1,x_2) = 0$, which means 
\begin{equation}
    F(x_1,x_2) \rho_{ss}(x_1,x_2,t)  - \sigma \nabla \rho_{ss}(x_1,x_2,t) = 0.
\end{equation}
This condition is called detailed balance and zero flux lead to equilibrium. Under this condition, the force can be regard as a negative gradient of a loss function $F(x_1, x_2) = - \nabla U(x_1,x_2)$, then the stationary distribution is calculated as
\begin{equation}
    \rho_{ss}(x_1,x_2) = \exp\left\{ - \frac{1}{\sigma} U(x_1,x_2) \right\}.
\end{equation}
For the second case, we have $\nabla \cdot J_{ss}(x_1,x_2) = 0$ and $J_{ss}(x_1,x_2) \neq 0$, current $\rho_{ss}(x_1,x_2)$ is called as Non-Equilibrium Stationary State (NESS). Under NESS, the force vector can be decomposed into two terms, which is
\begin{equation}
    F(x_1,x_2) = \frac{J_{ss}(x_1,x_2)}{\rho_{ss}(x_1,x_2)} + \frac{\sigma}{\rho_{ss}(x_1,x_2)} \nabla \rho_{ss}(x_1,x_2,t) = \frac{J_{ss}(x_1,x_2)}{\rho_{ss}(x_1,x_2)} + \sigma \nabla \ln \rho_{ss}(x_1,x_2).
\end{equation}
Consider the NESS as a Boltzmann-Gibbs form distribution $\rho_{ss}(x_1,x_2) = \exp{(-\widetilde{U}(x_1,x_2))}$ and $\widetilde{U}(x_1,x_2)$ is effective loss or non-equilibrium loss, then the decomposition reduce to
\begin{equation}
    F(x_1,x_2) = \frac{J_{ss}(x_1,x_2)}{\rho_{ss}(x_1,x_2)} - \sigma \nabla \widetilde{U}(x_1,x_2).
\end{equation}
However, in this case, there are no general analytic solution. So when the force vector is not a gradient of a analytic form loss function, most of the time we can only use numerical method to solve the NESS. So in the following section, we use a numerical method to solve the NESS with the drift term as the semi-gradient method. 

We defined $\frac{J_{ss}(x_1,x_2)}{\rho_{ss}(x_1,x_2)}$ as the effective flux and $- \sigma \nabla \widetilde{U}(x_1,x_2)$ as the effective gradient. The effective flux, effective gradient and force vector satisfies the parallelogram law. The loss landscape is defined as $\ln \rho_{ss}(x_1,x_2)$.

% 通常大家通过这个方式进行可视化，我们通过可视化来帮助理解：可视化直观帮助，直观带来理解。带解释性的研究的指向。

\newpage

\section{Proof of Theorems}
\label{sec::proof}
\begin{proof}[Proof of Lemma \ref{lem::existenceOfGloablMin}]
    From the Assumption \ref{assum:mdpAssum}, policy $\pi_1$ represents $\theta(a_1) \geq \theta(a_2)$, and policy $\pi_2$ represents $\theta(a_1) \leq \theta(a_2)$, $\theta(a_1)=\theta(a_2)$ is the policy boundary. Solving the following system of linear equation with policy $\pi_1$, 
    \begin{equation}
        \left \{
        \begin{aligned}
            & \phi(s_{\alpha})\theta(a_1) =  C + \gamma \phi(s'_{\alpha})\theta(a_1), \\
            & \phi(s_{\beta})\theta(a_2) = C + \gamma \phi(s'_{\beta})\theta(a_1).
        \end{aligned}
        \right.
    \end{equation}
    We have 
    \begin{equation}
        \theta_{\pi_1}(a_1) = \frac{C}{\phi(s_{\alpha}) - \gamma \phi(s'_{\alpha})} 
        \text{ and }
        \theta_{\pi_1}(a_2) = \frac{C}{\phi(s_{\beta})} + \frac{C\gamma \phi(s'_{\beta})}{\phi(s_{\beta}) \Bigl( \phi(s_{\alpha}) - \gamma \phi(s'_{\alpha}) \Bigr)}.
    \end{equation}
    This solution exist only when $\theta_{\pi_1}(a_1) \geq \theta_{\pi_1}(a_2)$, then we have
    \begin{equation}
        \left( 1-\frac{\gamma \phi(s'_{\beta})}{\phi(s_{\beta})} \right) \frac{1}{\phi(s_{\alpha}) - \gamma \phi(s'_{\alpha})} \leq \frac{1}{\phi(s_{\beta})} \Rightarrow \frac{\phi(s_{\beta})-\gamma \phi(s'_{\beta}) }{\phi(s_{\alpha})-\gamma \phi(s'_{\alpha})} \leq 1.
    \end{equation}
    Solving the following system of linear equation with policy $\pi_2$, 
    \begin{equation}
        \left \{
        \begin{aligned}
            & \phi(s_{\alpha})\theta(a_1) =  C + \gamma \phi(s'_{\alpha})\theta(a_2), \\
            & \phi(s_{\beta})\theta(a_2) = C + \gamma \phi(s'_{\beta})\theta(a_2).
        \end{aligned}
        \right.
    \end{equation}
    We have 
    \begin{equation}
        \theta_{\pi_2}(a_1) = \frac{C}{\phi(s_{\alpha})} +  \frac{C\gamma \phi(s'_{\alpha})}{\phi(s_{\alpha}) \Bigl( \phi(s_{\beta}) - \gamma \phi(s'_{\beta}) \Bigr)} \text{ and } \theta_{\pi_2}(a_2) = \frac{C}{\phi(s_{\beta}) - \gamma \phi(s'_{\beta})}.
    \end{equation}
    This solution exist only when $\theta_{\pi_1}(a_1) \geq \theta_{\pi_1}(a_2)$, then we have
    \begin{equation}
        \left( 1 - \frac{\gamma \phi(s'_{\alpha})}{\phi(s_{\alpha})} \right) \frac{1}{\phi(s_{\beta}) - \gamma \phi(s'_{\beta})} \leq \frac{1}{\phi(s_{\alpha})} \Rightarrow \frac{\phi(s_{\alpha})-\gamma \phi(s'_{\alpha})}{ \phi(s_{\beta})-\gamma \phi(s'_{\beta})} \leq 1.
    \end{equation}
\end{proof}

\begin{proof}[Proof of Lemma \ref{lem::continuitySemi}]
    We first define the set for the policy boundary as $\Omega := \{\theta \in \mathbb{R}^2 | \theta(a_1) = \theta(a_2)\}$, the set for policy $\pi_1$ as $\Omega^{\pi_1} := \{\theta \in \mathbb{R}^2 | \theta(a_1) > \theta(a_2)\}$ and the set for policy $\pi_2$ as $\Omega^{\pi_2} := \{\theta \in \mathbb{R}^2 | \theta(a_1) < \theta(a_2)\}$. It is easy to check the continuity of semi-gradient in non-boundary area, so here we only consider the $\theta \in \Omega$. Assume $h = [h_1, h_2]^{\text{T}}$ and $h_1 > h_2$, so we have $\theta + h \in \Omega^{\pi_1}$ and $\theta - h \in \Omega^{\pi_2}$. For the semi-gradient we have
    \begin{equation}
        \begin{aligned}
            & \lim_{\norm{h} \to 0} \left( F_{\text{Semi}}(\theta + h) - F_{\text{Semi}}(\theta - h) \right)
            = \lim_{\norm{h} \to 0} \left( F^{\pi_1}_{\text{Semi}}(\theta + h) - F^{\pi_2}_{\text{Semi}}(\theta - h) \right) \\
            = & \lim_{\norm{h} \to 0} \left(
            \begin{aligned}
                & \phi(s_{\alpha}) \left(  \gamma (h_1 + h_2) \phi(s'_{\alpha}) - 2 h_1 \phi(s_{\alpha})  \right) \\
                & \phi(s_{\alpha}) \left(  \gamma (h_1 + h_2) \phi(s'_{\beta}) - 2 h_2 \phi(s_{\beta})  \right)
            \end{aligned}
            \right) = 0.
        \end{aligned}
    \end{equation}
    so we have $\lim_{h \to 0} \abs{  F_{\text{Semi}}(\theta + h) - F_{\text{Semi}}(\theta - h)  } = 0$. For the residual gradient, we have
    \begin{equation}
        \begin{aligned}
            & \lim_{\norm{h} \to 0} \left( F_{\text{res}}(\theta + h) - F_{\text{res}}(\theta - h) \right )= 
            \lim_{\norm{h} \to 0} \left( F^{\pi_1}_{\text{res}}(\theta + h) - F^{\pi_2}_{\text{res}}(\theta - h) \right ) \\
            & = \left(
            \begin{aligned}
                & -\gamma \,\left(C\,\phi(s'_{\beta})+C\,\phi(s'_{\alpha})+\gamma \,{\phi(s'_{\beta})}^2\,\theta(a_1)+\gamma \,{\phi(s'_{\alpha})}^2\,\theta(a_1) \right. \\
                &~~~~~ \left. -\phi(s_{\beta})\,\phi(s'_{\beta})\,\theta(a_1)-\phi(s_{\alpha})\,\phi(s'_{\alpha})\,\theta(a_1)\right), \\ 
                & \phi(s_{\alpha})\,\left(C-\phi(s_{\beta})\,\theta(a_1)+\gamma \,\phi(s'_{\beta})\,\theta(a_1)\right) \\
                &~~~~~ -\left(\phi(s_{\beta})-\gamma \,\phi(s'_{\beta})\right)\,\left(C-\phi(s_{\beta})\,\theta(a_1) +\gamma \,\phi(s'_{\beta})\,\theta(a_1)\right) \\
                &~~~~~ +\gamma \,\phi(s'_{\alpha})\,\left(C-\phi(s_{\alpha})\,\theta(a_1)+\gamma \,\phi(s'_{\alpha})\,\theta(a_1)\right) 
            \end{aligned}
            \right) \neq 0
        \end{aligned}
    \end{equation}
\end{proof}

\begin{proof}[Proof of Theorem \ref{thm::SaddPointRange}]
    Define $\Omega = \{\theta \in \mathbb{R}^2 | \lambda \theta_{\pi_1} + (1-\lambda) \theta_{\pi_2}, \lambda \in (0,1) \}$. The parameter in both $\Omega$ and policy boundary is
    \begin{equation*}
        \theta_{\lambda^*} = \lambda \theta_{\pi_1} + (1-\lambda) \theta_{\pi_2} \text{ and } \lambda^* = \frac{\phi(s_{\beta})\,\left(\phi(s_{\alpha})-\gamma \,\phi(s'_{\alpha})\right)}{\gamma \,\left(\phi(s_{\beta})\,\phi(s'_{\alpha})-\phi(s'_{\beta})\,\phi(s_{\alpha})\right)}.
    \end{equation*}
    It is easy to verify $0<\lambda^*<1$ under the condition given by (1) and (2). Then we have
    \begin{equation*}
        \begin{aligned}
            \Omega \cup \{ \theta(a_2) \leq \theta(a_1) \} & = \{\theta \in \mathbb{R}^2 | \eta \theta_{\pi_1} + (1-\eta) \theta_{\lambda^*}, \eta \in [0,1) \} \\
            \Omega \cup \{ \theta(a_2) \geq \theta(a_1) \} & = \{\theta \in \mathbb{R}^2 | \eta \theta_{\pi_2} + (1-\eta) \theta_{\lambda^*}, \eta \in [0,1) \} 
        \end{aligned}
    \end{equation*}
    Calculate the two terms in the statement, we have
    \begin{equation*}
        \begin{aligned}
            & \frac{\langle (\theta_{\pi_1} - \theta_{\pi_2}), F^{\pi_2}_{\rm{semi}}(\theta_{\eta}) \rangle}{\norm{\theta_{\pi_1} - \theta_{\pi_2}} \norm{F^{\pi_2}_{\rm{semi}}(\theta_{\eta})}} 
            = \frac{\langle (\theta_{\pi_1} - \theta_{\pi_2}), F^{\pi_1}_{\rm{semi}}(\theta_{\eta}) \rangle}{\norm{\theta_{\pi_1} - \theta_{\pi_2}} \norm{F^{\pi_1}_{\rm{semi}}(\theta_{\eta})}} \\
            = & \frac{C^2\,\gamma \,\left(1-\eta\right)\,\left({\phi(s'_{\beta}})^2+{\phi(s'_{\alpha}})^2\right)\,{\left(\phi(s_{\beta})-\phi(s_{\alpha})-\gamma \,\phi(s'_{\beta})+\gamma \,\phi(s'_{\alpha})\right)}^2}{\left(\phi(s_{\beta})-\gamma \,\phi(s'_{\beta})\right)\,\Big(\phi(s_{\alpha})-\gamma \,\phi(s'_{\alpha})\Big)\,\left(\phi(s_{\beta})\,\phi(s'_{\alpha})-\phi(s'_{\beta})\,\phi(s_{\alpha})\right) \fM \fN}.
        \end{aligned}
    \end{equation*}
    where
    \begin{equation*}
        \begin{aligned}
            \fM & = \sqrt{\frac{C^2\,\left({\phi(s_{\beta}})^2\,{\phi(s'_{\beta}})^2+{\phi(s_{\alpha}})^2\,{\phi(s'_{\alpha}})^2\right)\,{\left(1-\eta\right)}^2\,{\left(\phi(s_{\beta})-\phi(s_{\alpha})-\gamma \,\phi(s'_{\beta})+\gamma \,\phi(s'_{\alpha})\right)}^2}{{\left(\phi(s_{\beta})\,\phi(s'_{\alpha})-\phi(s'_{\beta})\,\phi(s_{\alpha})\right)}^2}}, \\
            \fN & = \sqrt{\frac{C^2\,\gamma ^2\,\left({\phi(s_{\beta}})^2\,{\phi(s'_{\alpha}})^2+{\phi(s'_{\beta}})^2\,{\phi(s_{\alpha}})^2\right)\,{\left(\phi(s_{\beta})-\phi(s_{\alpha})-\gamma \,\phi(s'_{\beta})+\gamma \,\phi(s'_{\alpha})\right)}^2}{{\phi(s_{\beta}})^2\,{\phi(s_{\alpha}})^2\,{\left(\phi(s_{\beta})-\gamma \,\phi(s'_{\beta})\right)}^2\,{\Big(\phi(s_{\alpha})-\gamma \,\phi(s'_{\alpha})\Big)}^2}}.
        \end{aligned}
    \end{equation*}
    By the condition given in (1), which is $\phi(s_{\alpha})-\gamma \,\phi(s'_{\alpha}) > 0$ and $\phi(s_{\beta}) - \gamma \phi(s'_{\beta}) < 0$, we have $\phi(s_{\beta})\,\phi(s'_{\alpha})-\phi(s'_{\beta})\,\phi(s_{\alpha}) < 0$ and 
    \begin{equation*}
        \frac{\langle (\theta_{\pi_1} - \theta_{\pi_2}), F^{\pi_2}_{\rm{semi}}(\theta_{\eta}) \rangle}{\norm{\theta_{\pi_1} - \theta_{\pi_2}} \norm{F^{\pi_2}_{\rm{semi}}(\theta_{\eta})}} = \frac{\langle (\theta_{\pi_1} - \theta_{\pi_2}), F^{\pi_1}_{\rm{semi}}(\theta_{\eta}) \rangle}{\norm{\theta_{\pi_1} - \theta_{\pi_2}} \norm{F^{\pi_1}_{\rm{semi}}(\theta_{\eta})}} = 1.
    \end{equation*}
    By the condition given in (2), which is $\phi(s_{\alpha})-\gamma \,\phi(s'_{\alpha}) < 0$ and $\phi(s_{\beta}) - \gamma \phi(s'_{\beta}) > 0$, we have $\phi(s_{\beta})\,\phi(s'_{\alpha})-\phi(s'_{\beta})\,\phi(s_{\alpha}) > 0$ and
    \begin{equation*}
        \frac{\langle (\theta_{\pi_1} - \theta_{\pi_2}), F^{\pi_2}_{\rm{semi}}(\theta_{\eta}) \rangle}{\norm{\theta_{\pi_1} - \theta_{\pi_2}} \norm{F^{\pi_2}_{\rm{semi}}(\theta_{\eta})}} = \frac{\langle (\theta_{\pi_1} - \theta_{\pi_2}), F^{\pi_1}_{\rm{semi}}(\theta_{\eta}) \rangle}{\norm{\theta_{\pi_1} - \theta_{\pi_2}} \norm{F^{\pi_1}_{\rm{semi}}(\theta_{\eta})}} = -1.
    \end{equation*} 
\end{proof}

\newpage

\section{Additional Figures}
\label{sec::addition_figs}

\begin{figure}[h!]
\centering
\begin{subfigure}[t]{.4\textwidth}
  \centering
  % include first image
  \includegraphics[width=1\linewidth]{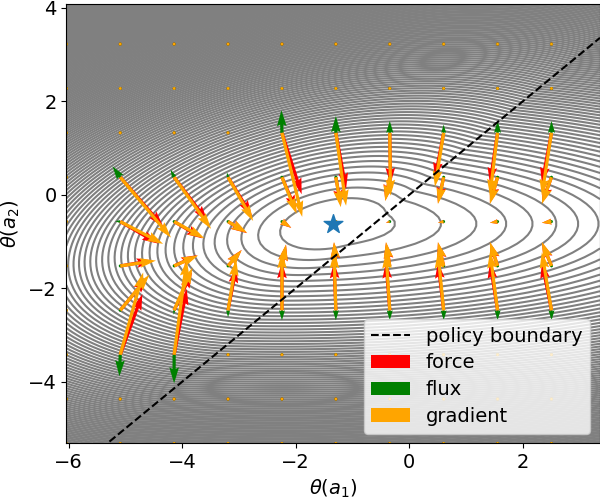}
  \caption{loss landscape}
  \label{fig:sub-first}
\end{subfigure}
\hspace{10mm}
\begin{subfigure}[t]{.4\textwidth}
  \centering
  % include first image
  \includegraphics[width=1\linewidth]{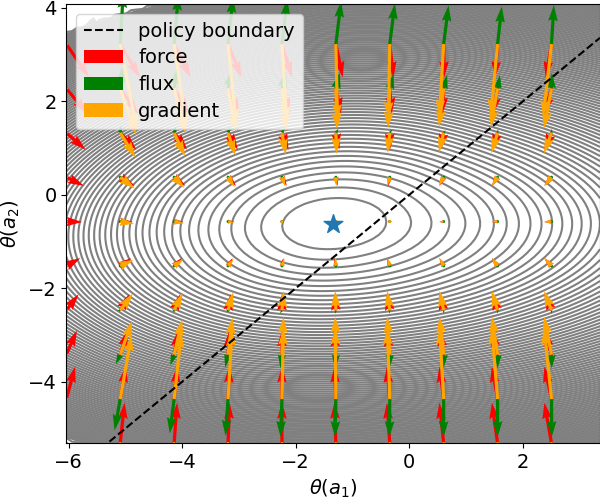}
  \caption{effective loss landscape}
  \label{fig:sub-first}
\end{subfigure}
\caption{Loss landscapes with mini-batch $\{(s_1, a_1, s_2, r)$ ,$(s_3, a_2, s_4, r)\}$. There is single critical point in the (effective) loss landscape. The two loss landscape is similar.}
\label{fig::loss_landscape_s1_a1_s2_s3_a2_s4}
\end{figure}

\begin{figure}[h!]
\centering
\begin{subfigure}[t]{.4\textwidth}
  \centering
  % include first image
  \includegraphics[width=1\linewidth]{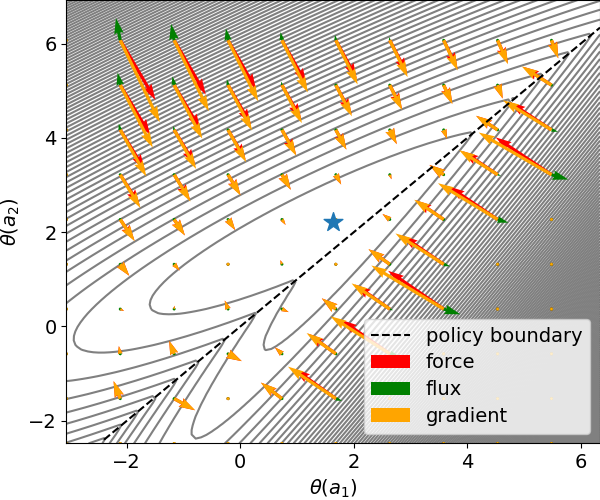}
  \caption{loss landscape}
  \label{fig:sub-first}
\end{subfigure}
\hspace{10mm}
\begin{subfigure}[t]{.4\textwidth}
  \centering
  % include first image
  \includegraphics[width=1\linewidth]{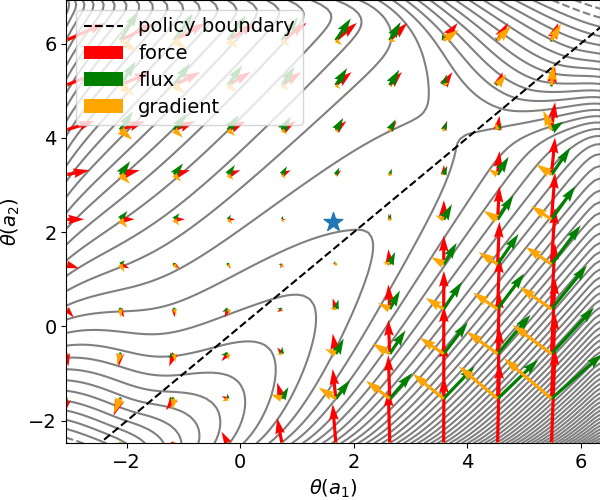}
  \caption{effective loss landscape}
  \label{fig:sub-first}
\end{subfigure}
\caption{Loss landscapes with mini-batch $\{(s_2, a_1, s_1, r)$, $(s_1, a_2, s_3, r)\}$.There is a single critical point in the (effective) loss landscape. Based on the contour's shape, this critical point in the effective loss landscape appears to be a saddle point. Due to the solution's proximity to the policy boundary (blue star), the critical point identified by the Fokker-Planck equation is slightly distant from the true critical point (blue star) because of the numerical error.}
\label{fig::loss_landscape_s2_a1_s1_s1_a2_s3}
\end{figure}

\begin{figure}[h!]
\centering
\begin{subfigure}[t]{.4\textwidth}
  \centering
  % include first image
  \includegraphics[width=1\linewidth]{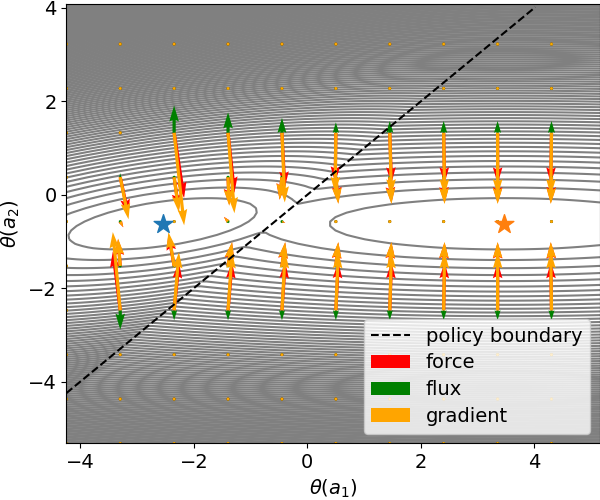}
  \caption{loss landscape}
  \label{fig:sub-first}
\end{subfigure}
\hspace{10mm}
\begin{subfigure}[t]{.4\textwidth}
  \centering
  % include first image
  \includegraphics[width=1\linewidth]{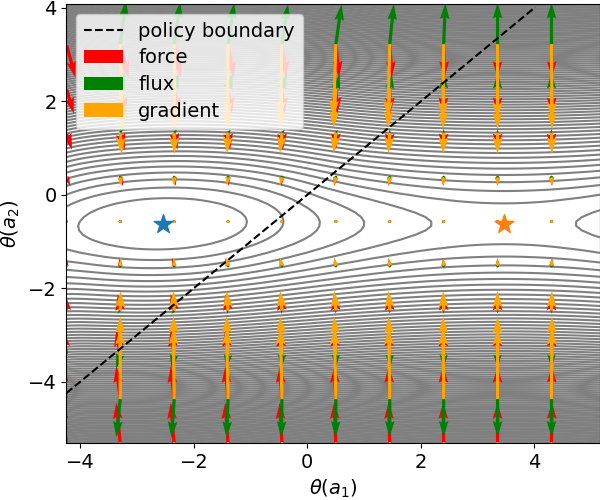}
  \caption{effective loss landscape}
  \label{fig:sub-first}
\end{subfigure}
\caption{Loss landscapes with mini-batch $\{(s_2, a_1, s_1, r)$, $(s_3, a_2, s_4, r)\}$. The $\theta_{\pi_1}$ in the effective loss landscape is a saddle point.}
\label{fig::loss_landscape_s2_a1_s1_s3_a2_s4}
\end{figure}

\begin{figure}[h!]
\centering
\begin{subfigure}[t]{.4\textwidth}
  \centering
  % include first image
  \includegraphics[width=1\linewidth]{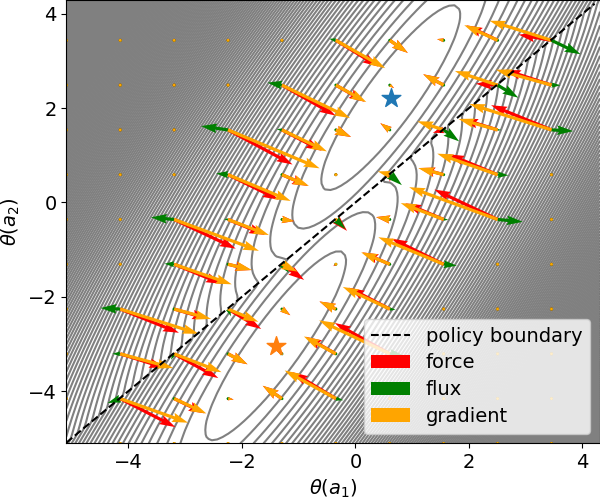}
  \caption{loss landscape}
  \label{fig:sub-first}
\end{subfigure}
\hspace{10mm}
\begin{subfigure}[t]{.4\textwidth}
  \centering
  % include first image
  \includegraphics[width=1\linewidth]{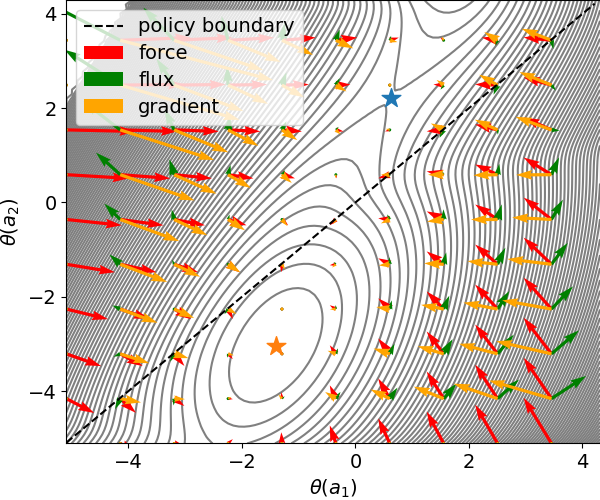}
  \caption{effective loss landscape}
  \label{fig:sub-first}
\end{subfigure}
\caption{Loss landscapes with mini-batch $\{(s_2, a_1, s_1, r)$, $(s_3, a_2, s_4, r)\}$. The $\theta_{\pi_2}$ in the effective loss landscape is a saddle point.}
\label{fig::loss_landscape_s3_a1_s1_s1_a2_s3}
\end{figure}

\begin{figure}[h!]
\centering
\begin{subfigure}[t]{.4\textwidth}
  \centering
  % include first image
  \includegraphics[width=1\linewidth]{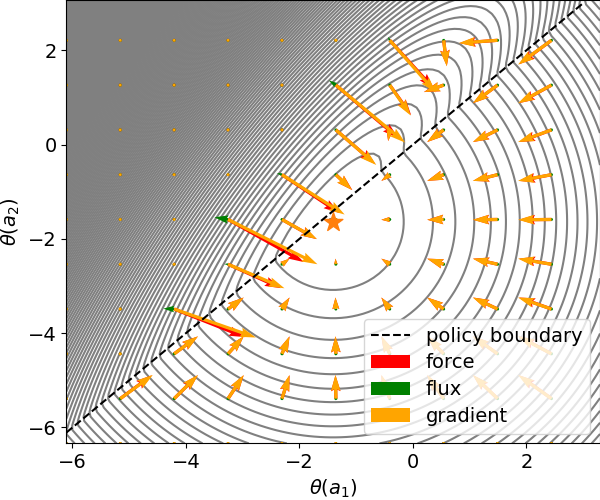}
  \caption{loss landscape}
  \label{fig:sub-first}
\end{subfigure}
\hspace{10mm}
\begin{subfigure}[t]{.4\textwidth}
  \centering
  % include first image
  \includegraphics[width=1\linewidth]{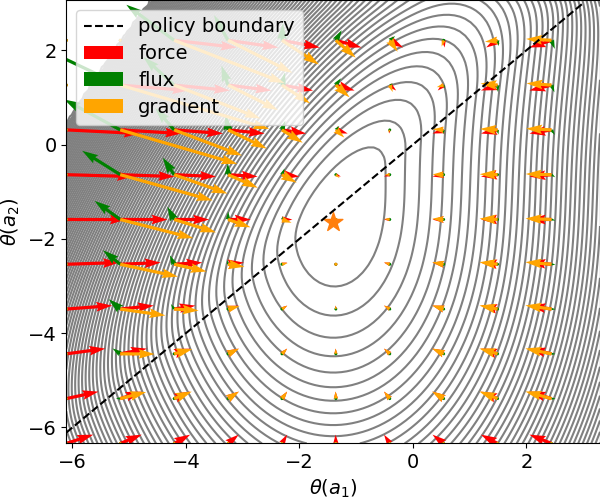}
  \caption{effective loss landscape}
  \label{fig:sub-first}
\end{subfigure}
\caption{Loss landscapes with mini-batch $\{(s_3, a_1, s_2, r)$, $(s_2, a_2, s_4, r)\}$. Only one critical point exist in (effective) loss landscape.}
\label{fig::loss_landscape_s3_a1_s1_s2_a2_s4}
\end{figure}

\begin{figure}[h!]
\centering
\begin{subfigure}[t]{.4\textwidth}
  \centering
  % include first image
  \includegraphics[width=1\linewidth]{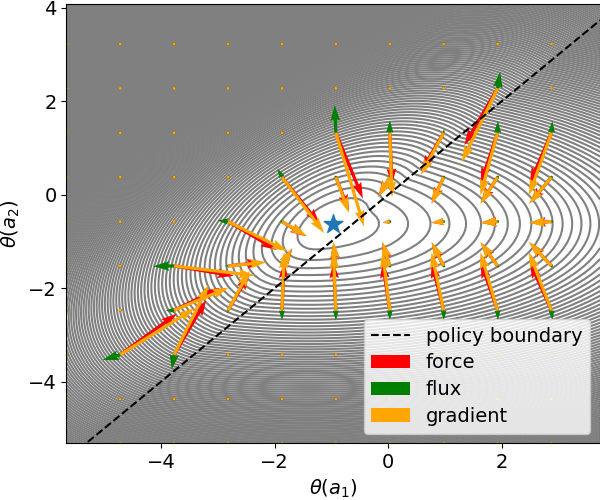}
  \caption{loss landscape}
  \label{fig:sub-first}
\end{subfigure}
\hspace{10mm}
\begin{subfigure}[t]{.4\textwidth}
  \centering
  % include first image
  \includegraphics[width=1\linewidth]{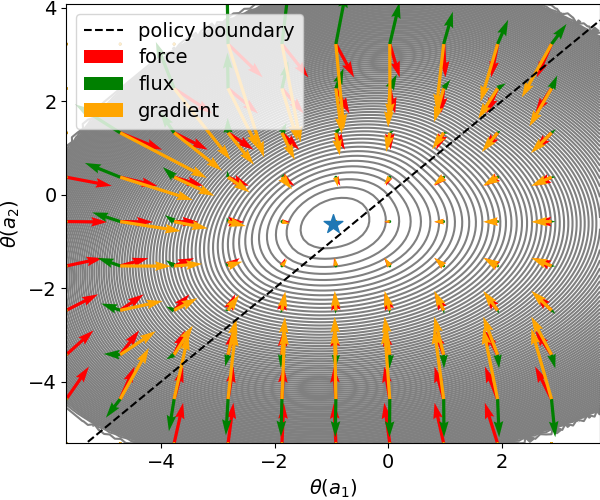}
  \caption{effective loss landscape}
  \label{fig:sub-first}
\end{subfigure}
\caption{Loss landscapes with mini-batch $\{(s_3, a_1, s_1, r)$,  $(s_3, a_2, s_4, r)\}$. Only one critical point exist in (effective) loss landscape.}
\label{fig::loss_landscape_s3_a1_s1_s3_a2_s4}
\end{figure}

\begin{figure}[h!]
\centering
% include first image
\includegraphics[width=0.4\linewidth]{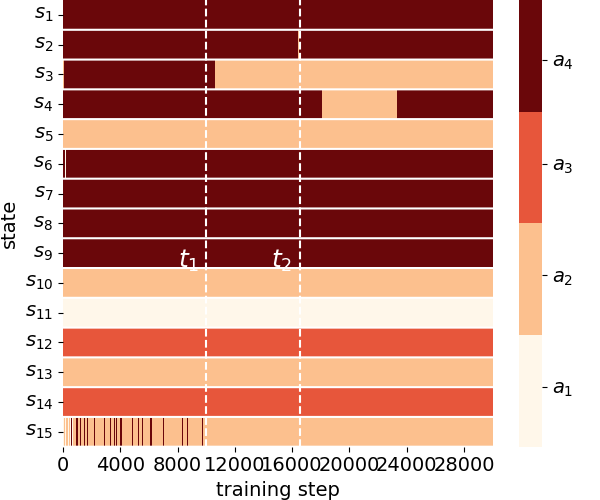}
\caption{The policy changes during training.}
\label{fig::policy_dynamics_full}
\end{figure}

\end{appendix}

\end{document}